\documentclass[letterpaper,twocolumn,10pt]{article}
\usepackage{usenix-2020-09}

\usepackage{amsmath,amsfonts,bm}

\def\eqref#1{equation~\ref{#1}}

\def\1{\bm{1}}

\def\eps{{\epsilon}}

\DeclareMathAlphabet{\mathsfit}{\encodingdefault}{\sfdefault}{m}{sl}
\SetMathAlphabet{\mathsfit}{bold}{\encodingdefault}{\sfdefault}{bx}{n}

\usepackage[utf8]{inputenc} 
\usepackage[T1]{fontenc}    
\usepackage{color}
\usepackage{xcolor}
   \definecolor{darkblue}{rgb}{0.0,0.0,0.75}
\usepackage{hyperref}       

   \hypersetup{colorlinks=true,linkcolor=darkblue,urlcolor=darkblue}
\usepackage{url}            
\usepackage{booktabs}       
\usepackage{nicefrac}       
\usepackage{microtype}      
\usepackage{dsfont}
\usepackage{amsthm}
\usepackage{amssymb}
\usepackage{amsmath}
\usepackage{enumitem}
\usepackage{pdfpages}
\usepackage{algpseudocode}
\usepackage{algorithm}
\usepackage{authblk}

\usepackage[utf8]{inputenc} 
\usepackage[T1]{fontenc}    
\usepackage{url}            
\usepackage{booktabs}       
\usepackage{nicefrac}       
\usepackage{microtype}      
\usepackage{algorithm}
\usepackage[algo2e]{algorithm2e}
\usepackage{multirow}
\usepackage{comment}
\usepackage{pgfplots}
\usepgfplotslibrary{statistics}

\usepackage{times}
\usepackage{amssymb}
\usepackage{amsmath}
\usepackage{comment}
\usepackage{float}
\usepackage{xspace}
\usepackage{tikz}
\usetikzlibrary{automata, positioning, arrows}

\usepackage{longtable}

\usepackage{subcaption}
\DeclareMathOperator{\bfx}{{\bf x}}

\DeclareMathOperator{\bfz}{{\bf z}}
\DeclareMathOperator{\bfw}{{\bf w}}

\DeclareMathOperator{\calD}{{\mathcal D}}
\DeclareMathOperator{\calX}{{\mathcal X}}

\DeclareMathOperator{\calS}{{\mathcal S}}
\DeclareMathOperator{\calY}{{\mathcal Y}}
\DeclareMathOperator{\calH}{{\mathcal H}}

\newtheorem{theorem}{Theorem}

\newtheorem{definition}{Definition}

\newtheorem{claim}{Claim}
\newtheorem{observation}{Observation}
\newtheorem{proposition}{Proposition}

\newcommand{\ke}[1]{\textcolor{red}{ke: #1}}
\def\bbE{\mathop{\mathbb{E}}}

\renewcommand{\figurename}{Figure.}

\def\calX{\mathcal{X}}

\def\calH{\mathcal{H}}
\def\calT{\mathcal{T}}

\def\calD{\mathcal{D}}
\def\calN{\mathcal{N}}
\def\calR{\mathcal{R}}
\def\bbE{\mathop{\mathbb{E}}}

\newcommand{\nap}{\textit{normalize-and-predict}\xspace}
\newcommand{\nandp}{\textit{N{\scriptsize \&}P}\xspace}

\newcommand{\mypara}[1]{\textbf{\textit{#1~}}}

\newcommand{\ma}[1]{\textbf{\color{orange}\emph{M: #1}}}
\newcommand{\yz}[1]{\textbf{\color{blue}\emph{Y: #1}}}
\newcommand{\effnorm}{\textproc{FastNorm}\xspace}

\title{Robust and Accurate Authorship Attribution via Program Normalization}

\author[1]{Yizhen Wang}
\author[2]{Mohannad Alhanahnah}
\author[1]{Ke Wang}
\author[1]{Mihai Christodorescu}
\author[2]{Somesh Jha}
\affil[1]{Visa Research}
\affil[2]{University of Wisconsin-Madison}
\affil[ ]{\{yizhewan, kewang, mihai.christodorescu\}@visa.com}
\affil[ ]{alhanahnah@wisc.edu,\ \ jha@cs.wisc.edu}

\begin{document}

\maketitle

\begin{abstract}
\ \ \ \ \ \ Source code attribution approaches have achieved remarkable accuracy thanks to the rapid advances in deep learning. However, recent studies shed light on 
their vulnerability to adversarial attacks. In particular, they can be easily deceived by adversaries who attempt to either create a forgery of another author or to mask the original author. To address these emerging issues, we formulate this security challenge into a general threat model, the \emph{relational adversary}, that allows an arbitrary number of the semantics-preserving transformations to be applied to an input in any problem space. Our theoretical investigation shows the conditions for robustness and the trade-off between robustness and accuracy in depth. Motivated by these insights, we present a novel learning framework, \nap (\nandp), that in theory guarantees the robustness of any authorship-attribution approach. We conduct an extensive evaluation of \nandp in defending two of the latest authorship-attribution approaches against state-of-the-art attack methods.
Our evaluation demonstrates that \nandp improves the accuracy on adversarial inputs by as much as 70\% over the vanilla models. More importantly, \nandp also increases robust accuracy to 45\% higher than adversarial training while running over 40 times faster. 

\end{abstract}
\section{Introduction}

Source code authorship attribution, the problem of identifying the author of a computer program, has been receiving increased attention from the security community~\cite{DL-CAIS,De-Anonymizing,alsulami2017source,survey-Authorship-Attribution,Frantzeskou2006}. On the one hand, being able to identify authors of source code poses a privacy risk for programmers who wish to remain anonymous. For example, contributors to open-source projects may hide their identity to keep their side activities private, or else programmers who prefer to maintain their anonymity for their participation in projects with political implications. On the other 
hand, source code attribution is useful for software forensics and security analysts, especially for identifying malware authors. 
Although a careful adversary may only leave binaries, others could leave behind scripts or source code in a compromised system for compilation. 

\begin{figure}[t]
\centering
\includegraphics[width=0.5\textwidth]{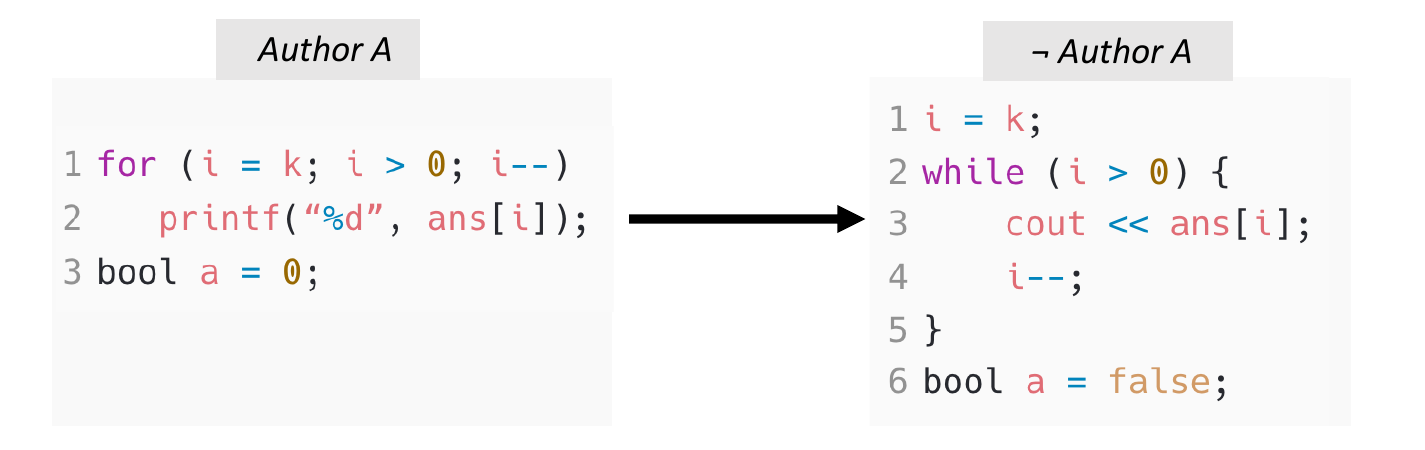}
\caption{An example of source code attribution attack. Three transformations are applied to the initial code written by Author A: (1) \emph{for-to-while} (2) \emph{printf-to-cout}, and (3) \emph{0/1-to-false/true}. The resulted code on the right is attributed to a different author by the ML model.}
\label{fig:motivation}
\end{figure}

While recent advances in deep learning methods have empowered authorship attribution approaches to achieve remarkable accuracy, their robustness against adversarial attacks has been called into question. Simko et al.~\cite{simko2018recognizing} discovered existing authorship attribution approaches are vulnerable to adversarial attempts at creating forgeries. Quiring et al.~\cite{USENIX-authorthisp} presented an attack method against authorship attribution of source code. 
Drawing from adversarial machine learning research, they develop a variant of Monte-Carlo tree search to transform the original code samples in a way that consistently fools the state-of-the-art authorship attribution approaches. 
As an example, \autoref{fig:motivation} illustrates such an attack, which misleads the predication of the attribution model from author $A$ to a different author. This attack applies three source-code transformations. The first changes the
\emph{for}-loop to a \emph{while}-loop; the second replaces the \emph{printf} statement with a \emph{cout} statement; and the last converts the integer representation of Boolean variables to true/false. 

These attacks not only compromise the performance of source code attribution approaches but also raise two general, fundamental questions regarding machine learning models: (1) can a learning framework produce a provably robust model against
transformations that preserve the semantics of the input data (the original source code)? and (2) if so, what is the highest accuracy that this model can achieve?

\paragraph{Key Insights}
This paper presents a novel learning framework, \nandp\footnote{\nandp{} is the abbreviation of \nap.},
that guarantees the robustness of machine learning models against semantics-preserving transformations.
Our insights is to define a unique normal form of each data point with respect to the transformations. Subsequently, 
models are re-trained on the normal form of training data to make predictions on that of test data. Thus, whenever a data point is manipulated at test-time, \nandp automatically reverts it back to its normal form for which models always produce a constant prediction, making models adversarially resistant.
We note this normalization paradigm 
is leveraged in the domain of network intrusion detection~\cite{normalizer-2001}, where a stream of network packets is patched up by a normalizer to eliminate potential ambiguities before the traffic is seen by the monitor, thus removing evasion opportunities. Investigation of adversarial robustness in the problem space (in our case, author attribution) has been emphasized in~\cite{pierazzi2020problemspace}.

Compared to adversarial training, the standard approach to enhancing model robustness, \nandp offers two crucial advantages. First and foremost, adversarial training in principle only achieves high robust accuracy when the adversarial example in the training loop maximizes the loss. However, finding such adversarial examples through logic transformations in the problem space is substantially more expensive than $\ell_p$-norm bounded ones in the feature space. This is because neither the generation nor the sanity check of adversarial examples in the problem space can be directly performed on GPUs in tandem with the update of model parameters. Taking into account that the complexity of adversarial training is proportional to the number of training steps (i.e., the number of adversarial examples to be found), such a computation procedure can be exceedingly inefficient. In contrast, finding the normal form, while also incurring computational overhead, is a one-time process, resulting in a far more efficient algorithm.

Second, existing adversarial training algorithms focus on parametric methods (i.e., neural networks and linear classifiers) with a well-established paradigm: minimizing the training loss of models on adversarial examples~\cite{madry2018towards}. However, non-parametric methods (e.g., nearest neighbor, decision tree, and random forest) have no gradients, which render the existing gradient-based attacks inapplicable. As a result, adversarial training appears to be ineffective for those methods~\cite{yang2019adversarial, wang2018analyzing}.

We instantiate \nandp to the setting of source code authorship attribution attack and extensively evaluate it. In defending the two latest attribution approaches~\cite{rf-caliskan,rnn-abuhamad}
against the attack methods proposed by Quiring et al.~\cite{USENIX-authorthisp},
\nandp achieves significantly improved robust accuracy compared to adversarial training while incurring far less computational overhead.

\paragraph{Contributions} We make the following contributions:

\begin{itemize}
    \item We propose a general threat model, the relational adversary, that stems from logical relations and captures the security challenges facing semantics-preserving input transformations over any input space. (\autoref{sec:relationaladversary})
    \item We analyze the condition for robust prediction and the fundamental robustness-accuracy trade-off in the presence of a relational adversary. (\autoref{sec:nandp})
    \item Inspired by the theoretical insights, we propose \nap, a learning framework that yields provably robust models to relational adversaries. (\autoref{sec:nandp}) We also propose \effnorm, a computationally efficient algorithm that normalizes the training and test inputs in our \nandp pipeline. (\autoref{sec:effnorm}) 
    \item We instantiate the \nandp framework for the authorship attribution tasks in~\cite{USENIX-authorthisp} and implement \effnorm with respect to the attack transformations in~\cite{USENIX-authorthisp}. Our empirical valuation in \autoref{sec:exp} shows that models obtained from \nandp have significantly improved accuracy under attacks. The test accuracy on the adversarial inputs is improved by >70\% over the vanilla model and >45\% over adversarial training for the attack in~\cite{USENIX-authorthisp}. (\autoref{table:acc})  The training time overhead is also substantially shorter in \nandp (<12 hrs) than adversarial training (>20 days). (\autoref{table:time})
\end{itemize}

The remainder of this paper is structured as follows. We begin by introducing the threat models for both the general relational adversary and the source code attribution attacks in \autoref{sec:threatmodel}. We present our \nandp framework in \autoref{sec:nandp} and compare our approach with adversarial training in \autoref{sec:comparison} followed by a detailed description for the implementation of a computationally efficient normalizer in \autoref{sec:effnorm}. We proceed to introduce experiments conducted to evaluate our method and discuss the insights they provide in \autoref{sec:exp}. \autoref{sec:related_work} presents related work, and \autoref{sec:conclusion} concludes.

\section{Threat Model} \label{sec:threatmodel}
In this section, we introduce the threat models for our theoretical study and empirical evaluation, respectively. We start with the definition of adversarial attack for general machine learning tasks as well as the notion of robust accuracy -- the key performance metric to any robust learning task. 
Next, we describe the source code attribution attack proposed by Quiring et al.~\cite{USENIX-authorthisp} against authorship attribution ML models. This attack will be the main setting of our empirical evaluation.
Last, we abstract the authorship evasion attack to a general notion of \emph{relational adversary}. We investigate the condition for robustness against this general attack and use the insights to build robust ML models against source code attribution.

\subsection{Test-time Adversarial Attack}
A typical machine learning classifier is a function $f:\calX\to\calY$ that takes an input $\bfx \in \calX$ with ground-truth label $y\in \calY$ and returns a predicted class label $f(\bfx)\in \calY$. 
A test-time adversary replaces a clean test input $\bfx$ with an adversarially manipulated input $A(\bfx)$, where $A(\cdot)$ represents the attack algorithm. We consider an adversary who wants to maximize the classification error rate over the data distribution $\calD$:
$$
\bbE_{(\bfx,y)\sim \calD} \mathds{1}(f(A(\bfx))\neq y).
$$
Let $\calS(\bfx)$ denote the feasible set from which the attacker can choose the adversarial example (AE), i.e., $A(\bfx) \in \calS(\bfx), \forall A(\cdot), \bfx$.
\begin{definition}[Robustness and robust accuracy]
A classifier $f$ is robust at an input $\bfx$ \emph{iff} $\bfx' \in \calS(\bfx) \implies f(\bfx') = f(\bfx)$. Furthermore, $f$ is robustly accurate at $\bfx$ \emph{iff} $\bfx' \in \calS(\bfx) \implies f(\bfx) = f(\bfx') = y$. The robust accuracy of $f$ over a distribution $\calD$ is
$$
\mathbb{E}_{(\bfx, y)\sim \calD}\mathds{1}(\mbox{$f$ is robustly accurate at $\bfx$})
$$
\end{definition}

\subsection{Source Code Attribution Attack}
In the source code attribution attack, the victim model takes in a piece of source code as the input from a domain $\calX$ and attributes the author of code as the output. The adversary is equipped with a set of semantics-preserving transformations $\calT$. Each element $T\in\calT$ is a function $\calX\to\calX$ that maps a piece of code to a semantics-preserving variant. 
During test time, the adversary can manipulate the test input by applying any sequence of transformations in $\calT$. The ultimate goal of the attack is fooling the source code authorship identification models to predict a wrong author for the test input. As an exemplar, the attacks proposed in~\cite{USENIX-authorthisp,SCAD-TIFS} performs a spectrum of source code transformations: API usage, variable declaration, I/O style, and control flow. We note that the transformations are all in the \emph{problem space}, and the transformed programs still compile and run with the same behavior. 
An attack is successful if the model attributes the transformed code to a different author than the original. 

In terms of the adversary's knowledge over the learning pipeline, the attacks in~\cite{USENIX-authorthisp, SCAD-TIFS} are black-box. The attacker does not know the model structure or parameters, but can query the model prediction for any legitimate test inputs. This querying ability allows the attacker to run a Monte-Carlo tree search to transform the test input until misclassification occurs.

As the learner and defender of the authorship attribution task, our goal is to learn a robust and accurate authorship attribution model that can still correctly classify the adversarially manipulated input as much as we can.
In this paper, we evaluate the model performance by its robust accuracy. 
A model is robustly accurate for a code piece if it can still correctly identify the original user after adversarial transformations.
In our empirical evaluation, we consider the same black-box attack strategy in~\cite{USENIX-authorthisp}. However, our theoretical results also hold for the stronger white-box attack setting in which the attacker has full knowledge of the learning pipeline.

\subsection{Relational Adversary}
\label{sec:relationaladversary}
As a natural generalization to source code attribution, we extend the threat model to any problem spaces subject to any given set of transformations. We formulate such security challenge as a \emph{relational adversary} as follows.

\noindent \mypara{Logical Relation.}
A relation $\mathcal{R}$ is a set of input pairs, where each pair $(\bfx,\bfz)$ specifies a transformation of input $\bfx$ to output $\bfz$. 
We write $\bfx \rightarrow_{\mathcal{R}} \bfz$ iff
$(\bfx,\bfz) \in \mathcal{R}$. We write $\bfx
\rightarrow_{\mathcal{R}}^* \bfz$ iff $\bfx = \bfz$ or there exists
$\bfz_0,\bfz_1,\cdots,\bfz_k$ ($k > 0$) such that $\bfx = \bfz_0$,
$\bfz_i \rightarrow_{\mathcal{R}} \bfz_{i+1}$ ($0 \leq i < k$) and
$\bfz_k = \bfz$. In other words, $\rightarrow_{\mathcal{R}}^*$ is the
reflexive-transitive closure of $\rightarrow_{\mathcal{R}}$.

\begin{definition}[relational adversary] An adversary is said to be $\calR$-relational if $\calS(\bfx) = \{\bfz|\bfx\to_{\calR}^*\bfz\}$, i.e. each element in $\calR$ represents an admissible transformation, and the adversary can apply arbitrary number of transformation specified by $\calR$.   
\end{definition}

In the source code attribution case, a pair of code pieces $(\bfx, \bfz)$ is in the relation $\calR$ if and only if $T(\bfx) = \bfz$ for some transformation $T\in \calT$.
\section{\nandp\ -- A Provably Robust Learning Framework}
\label{sec:nandp}

\begin{figure*}[h!]
\centering
\includegraphics[width=0.8\textwidth]{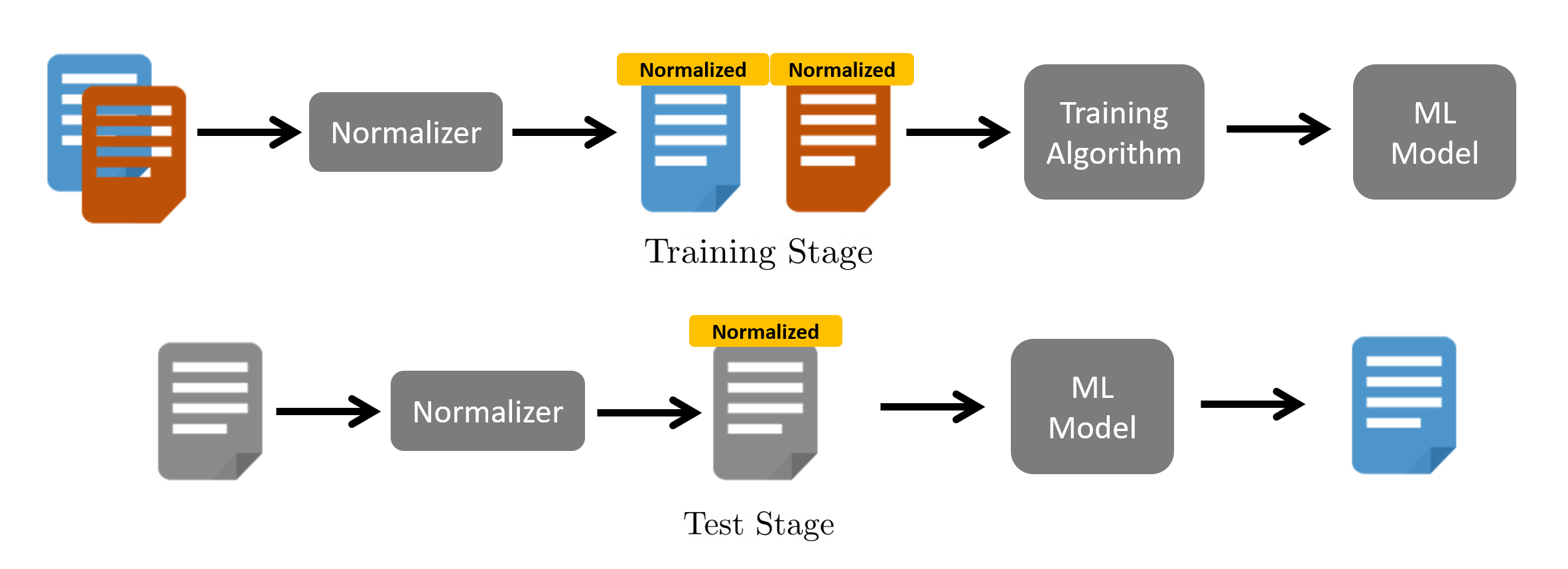}
\caption{Schematic illustration of the \nandp pipeline. Files in blue and red colors belong to different classes, e.g. benign and malicious. A file is grey in color if its class label is unknown. During the training time, the learner collects training inputs with labels, passes the inputs to the normalizer $\calN$, and uses the normalized training inputs to learn an ML model. At test time, the user encounters an input with unknown label, passes the input to $\calN$ again and uses the ML model to predict the label of the normalized test input.} 
\label{fig:nap}
\end{figure*}

In this section, we introduce \nandp, a learning framework which learns and predicts over normalized training and test inputs. \nandp works on a simple and intuitive premise: if the original and the adversarial test can be processed into the same form, then a classifier trained and tested on such form will be robust. We first identify the necessary and sufficient condition for robustness, and propose a normalization procedure that makes \nandp provably robust to $\calR$-relational adversaries.
Finally, we analyze the performance of \nandp: since \nandp guarantees robustness, the analysis will focus on robustness-accuracy trade-off and provide an in-depth understanding to causes of such trade-off.  

\begin{table*}
\centering
\caption{Comparison of training objective and test output for standard risk minimization learning scheme, \nandp and adversarial training; $f^*$ is the minimizer of the training objective.}
\renewcommand{\arraystretch}{1.5}
\begin{tabular}{p{1.2cm}c@{\qquad}c@{\qquad}c} 
\toprule
& No Defense & Normalize-and-Predict & Adversarial Training \\
\midrule
Train 
& $\min\limits_{f} \sum\limits_{(\bfx,y) \in D}\ell(f, \bfx, y)$
& $\min\limits_{f} \sum\limits_{(\bfx,y) \in D}\ell(f, \calN(\bfx), y)$ 
& $\min\limits_{f}\max\limits_{A(\cdot)} \sum\limits_{(\bfx,y) \in D}\ell(f, A(\bfx), y)$ \\
Test 
& $f^*(x)$
& $f^*(\calN(x))$
& $f^*(x)$\\
\bottomrule
\end{tabular}
\label{table:comp}
\end{table*}

\subsection{An Overview of the \nandp Framework}
Figure~\ref{fig:nap} illustrates the \nandp learning pipeline. The learner first specifies a normalizer $\calN:\calX \to \calX$.
We call $\calN(\bfx)$ the `normal form' of input $\bfx$.
The learner then both trains the classifier and predicts the test label over the normal forms instead of the original inputs. Let $D$ denote the training set. In the empirical risk minimization learning scheme, the learner will now solve the following problem
\begin{equation}
\min_{f} \sum_{(\bfx,y)\in D}\ell(f, \calN(\bfx), y),
\end{equation}
and use the minimizer $f^*$ as the classifier.
During test time, the model will predict $f^*(\calN(\bfx))$. Table~\ref{table:comp} compares the \nandp learning pipeline to the normal risk minimization.

\subsection{Finding the Normalizer}
\label{sec:master-normalizer}
The normalizer $\calN$ is crucial to achieving robustness: intuitively, if $\bfx$ and its adversarial example $\bfx_{adv}$ share the same normal form, then the prediction will be robust. Meanwhile, a constant $\calN$ is robust, but has no utility as $f(\calN(\cdot))$ is also constant. Therefore, we seek an $\calN$ that performs only the necessary normalization for robustness with a minimal impact on accuracy. 

In order to study the fundamental requirement of robust classification and the theoretical limit of robust accuracy, we first construct the {\bf relational graph} $G_{\calR}=\{V,E\}$ of $\calR$: the vertex set $V$ contains all elements in $\calX$; the edge set $E$ contains an edge $(\bfx, \bfz)$ \emph{iff} $(\bfx, \bfz) \in \calR$. Then, a directed path exists from $\bfx$ to $\bfz$ \emph{iff} $\bfx\to_{\calR}^* \bfz$. We derive the following necessary and sufficient condition for robustness under \nandp in \autoref{thm:robustcondition}, and subsequently a normalizer $\calN$ in \autoref{def:normal} that guarantees robustness.
\begin{observation}[Condition for Robustness]
\label{thm:robustcondition}
Let $C_1, \cdots, C_k$ denote the weakly connected components (WCC) in $G_{\calR}$. A classifier $f$ is robust for all $\bfx \in C_i$ \emph{iff} $f(\bfx)$ returns the same label for all $\bfx \in C_i$.
\end{observation}
\begin{definition}[Choice of Normalizer with Robustness Guarantee]
\label{def:normal}
Let $\calN$ be a function that maps an input $\bfx\in C_i$ to any deterministic element in $C_i$. Then $f(\calN(\cdot))$ is robust to $\calR$-relational adversaries.
\end{definition}  

\subsection{Robustness-Accuracy Trade-off}
\label{subsec:trade-off}
\mypara{Optimal Accuracy under \nandp.} 
Following the convention in machine learning, we use $\mu(\bfx)$ to denote the probability mass of input $\bfx$ drawn from the distribution, and use $\eta(\bfx,l)\,$$=$$\,\Pr(y\!\!\!=\!\!\!l|\bfx)$ to denote the probability that $\bfx$ has label $l$. \footnote{The label of an input $\bfx$ may also be probabilistic in nature.  For example, a ransomware and a zip tool may have the same static feature vector $\bfx$. The label of a randomly drawn $\bfx$ is probabilistic, and the probability depends on the frequency that each software appears.} 
Then the optimal robust accuracy using \nandp, denoted by $Acc_{\calR}^*$, is
$$
\sum_{C_i} \max_{l\in\calY}\sum_{\bfx\in C_i}\mu(\bfx)\eta(\bfx,l),
$$
which happens when
$$f(\calN(\bfx)) = \arg\max_{l\in\calY} \sum_{\bfx\in C_i}\mu(\bfx)\eta(\bfx,l)$$ for $\bfx\in C_i$. 
Intuitively, $f$ shall assign the most likely label of random samples in $C_i$ to all $\bfx\in C_i$. 

\mypara{Price of Robustness.} In \nandp, the optimal robust accuracy depends on $\calR$. We then observe the following robustness-accuracy trade-off: as the complexity of $\calR$ increases, we may lose accuracy for forcing models to make invariant predictions, and the accuracy loss is the price of robustness. 
\begin{observation}[Robustness-accuracy trade-off]
\label{thm:tradeoff}
Let $\calR'$ and $\calR$ be two relations s.t. $\calR' = \calR\bigcup\{(\bfx,\bfz)\}$, i.e. $\calR'$ allows an extra transformation from $\bfx$ to $\bfz$ than $\calR$. Let $C_{\bfx, \calR}$ denote the WCC in $G_{\calR}$ that contains $\bfx$, and $l_{C}$ be the most likely label of inputs in a WCC $C$.
Then $Acc^*_{\calR'} - Acc^*_{\calR} \leq 0$ for all $\calR, \calR'$ pairs, and the equality only holds when $l_{C_{\bfx,\calR}} = l_{C_{\bfz,\calR}}$.
\end{observation}

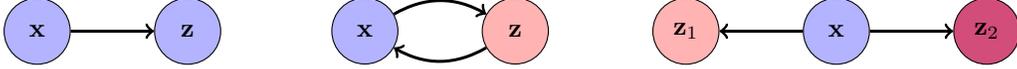
\begin{figure*}
\centering
\begin{tikzpicture}[node distance=2cm]
\node[state, fill=blue!30] (A) {$\bfx$};
\node[state, fill=blue!30] (B) [right of = A] {$\bfz$};
\path[->, line width=0.4mm] (A) edge node {}(B);
\end{tikzpicture}
\hspace{1.2cm}
\raisebox{-0.2em}{
\begin{tikzpicture}[node distance=2cm]
\node[state, fill=blue!30] (A) {$\bfx$};
\node[state, fill=red!30] (B) [right of = A] {$\bfz$};
\path[->, line width=0.4mm] (A) edge [bend left]node {}(B);
\path[->, line width=0.4mm] (B) edge [bend left]node {}(A);
\end{tikzpicture}}
\hspace{1.2cm}
\begin{tikzpicture}[node distance=2cm]
\node[state, fill=blue!30] (A) {$\bfx$};
\node[state, fill=red!30] (B) [left of = A] {$\bfz_1$};
\node[state, fill=purple!70] (C) [right of = A] {$\bfz_2$};
\path[->, line width=0.4mm] (A) edge node {}(B);
\path[->, line width=0.4mm] (A) edge node {}(C);
\end{tikzpicture}
\caption{Relations with different robustness-accuracy trade-off. Different node colors indicate different most likely labels. Appendix~\ref{subsec:labels} gives a detailed explanation on why semantics-preserving transformations can still change the labels of data. {\bf Left:} \nandp preserves natural accuracy; {\bf Middle:} \nandp preserves robust accuracy; {\bf Right:} \nandp causes suboptimal robust accuracy: suppose $\mu(\bfx)=0.02, \mu(\bfz_1)\!=\!\mu(\bfz_2)\!=\!0.49$, and $\eta$ is deterministic. \nandp predict the same label and thus has accuracy at most $0.49$, while the highest robust accuracy is $0.98$ by predicting the true label for $\bfz_1$ and $\bfz_2$.}
\label{fig:tradeoff}
\end{figure*}

The intuition is that the extra edge on the relation graph may join two connected components which are otherwise separate. As a result, models under \nandp will predict the same label for the two components, thus their accuracy on one component will drop if two components have different labels.

We further characterize three different levels of trade-offs (\autoref{fig:tradeoff}).
First, if two inputs $\bfx,\bfz$ have the same most likely label on $\calD$, then the optimal accuracy under \nandp is the same as before normalization, in other words, robustness is obtained \emph{for free}.
Second, if both $(\bfx,\bfz)$ and $(\bfz,\bfx)$ are in $\calR$ but $\bfx,\bfz$ have different most likely labels, then the model with the highest natural accuracy, which predicts the most likely label of $\bfx$ and $\bfz$ respectively, does not have any robustness. In contrast, \nandp achieves the optimal robust accuracy by predicting a \emph{single} label --- the most likely label of samples in $\{\bfx,\bfz\}$ --- for both $\bfx$ and $\bfz$.
Third, if $\bfx$ can only be one-way transformed to two inputs $\bfz_1, \bfz_2$ with different most likely labels, then \nandp may have suboptimal robust accuracy. An absolutely robust classifier need to predict the same label for $\bfx,\bfz_1$ and $\bfz_2$, while the classifier with the highest robust accuracy should predict the mostly likely labels for $\bfz_1$ and $\bfz_2$ if $\bfz_1,\bfz_2$ appear more frequently than $\bfx$.

\section{Comparing \nandp with Adversarial Training}
\label{sec:comparison}

Adversarial training is arguably the most widely acknowledged defense mechanism to adversarial attack. In principle, adversarial training tries to find the model that optimizes the min-max objective shown in \autoref{table:comp}. Such a model will have the highest robust accuracy by definition. However, for relational adversaries especially in the domain of program analysis and system security, the standard adversarial training suffers from its limitation in practice. In this section, we will show two major pain points of adversarial training and explain how \nandp can avoid or alleviate them. In addition, we will show a curious case in which the objectives of the two approaches align with each other, and also discuss a combination of the two approaches for the benefit from both worlds.

\subsection{Limitation of Adversarial Training}
\mypara{Prohibitive Computational Cost.}
In practice, the standard adversarial training procedure replaces the training input with an adversarial example in each training step and update the model with respect to loss over the adversarial example. 
The performance of adversarial training depends on the quality of the adversarial examples in the training loop: a stronger attack that finds adversarial example with higher loss typically leads to better models.
However, we show in \autoref{thm:pspace} in Appendix that the inner maximization problem is in general computationally infeasible for relational adversaries. While fast heuristics for $\ell_p$-norm constrained attacks exist for visual recognition tasks~\cite{shafahi2019adversarial, wong2020fast}, attacks involving code transformations are considerably slower. For example, the Monte-Carlo Tree Search (MCTS) attack in~\cite{USENIX-authorthisp} takes minutes to generate one adversarial example. Since the number of adversarial examples needed is proportionally to the number of model update steps, using adversarial training for code transformation attacks is prohibitively time-consuming.
Furthermore, the computation in the MCTS attack is often CPU-intensive; adding this procedure in the training loop will likely create a bottleneck for the original deep learning pipeline optimized for GPU.

\mypara{Incompatibility with Existing Learning Methods.}
The adversarial training algorithm plays a min-max game between the learner and an assumed attacker, and thus increases the robustness of the model after rounds of updates. However, this approach implicitly assumes the original model training procedure updates the model parameters in an iterative manner. While this is true for deep learning models trained with gradient descent, most of the non-parametric models, e.g. random forests, do not satisfy this assumption. The ML community is aware of this limitation and has come up with defenses for non-parametric models~\cite{yang2019adversarial, wang2018analyzing}. However, these approaches often require computing the collision between classes upon adversarial transformations, which does not scale well for multi-class classification and large data sets.

\subsection{Advantage of \nandp}
\nandp avoids the pain points of adversarial training by design. The normalization procedure is independent to the original training routine. Although normalization may be time-consuming depending on the complexity of the input and relation, it only needs to be performed \emph{once}. Meanwhile, since the normalization is independent to the original learning procedure, the normalizer can readily combine with any existing learning approach. In our empirical evaluation, we successfully accommodate both the random forest and the recurrent neural net approach into our \nandp framework with minimum interruption to the original learning script. We also show in Appendix~\ref{subsec:modelcapcity} that \nandp can reduce the model complexity required for high robust accuracy.

\subsection{Equivalence and Synergy between the Two Approaches}
\nandp and adversarial training have different objectives: the former tries to find an accurate model with provable robustness guarantee, while the latter directly optimizes for robust accuracy. Intriguingly, we find a common condition for program transformation under which the two objectives actually align with each other such that the optimizing these two objectives leads to identical robust accuracy. 

\begin{definition}[Reversible relation]
\label{def:reversiblerelation}
A relation $\calR$ is reversible iff $\bfx\to_{\calR^{*}}\bfz$ implies $\bfz\to_{\calR^{*}}\bfx$ and vice versa.
\end{definition}

\begin{theorem}[Preservation of robust accuracy]
\label{thm:equiv}
Let $f^*$ be the classifier that minimizes the objective of \nandp over the data distribution $\calD$, and let $f^*_{adv}$ minimize the objective of adversarial training over $\calD$.
Then, in principle, $f^*(\calN_{\calR}(\cdot))$ and $f^*_{adv}$ have the same optimal robust accuracy if $\calR$ is reversible.
\end{theorem}

The proof can be found in Appendix~\ref{sec:equivproof}. In essence, \autoref{thm:equiv} is a generalization of the second scenario in \autoref{fig:tradeoff}, in particular, we extend the same principle applied to $(\bfx,\bfz)$ to all possible pairs of inputs in the relational graph induced by $\calR$. Note that the relation between semantics-preserving pairs of programs is reversible by default: if $\bfz$ is $\bfx$'s adversarial example, then $\bfx$ is also naturally an adversarial choice of $\bfz$.

Motivated by the above observations, we propose a natural synergy of the two approaches: for a relation $\calR$ and a reversible subset $\calR'\subseteq \calR$, we can first normalize the input with respect to $\calR'$, and then adversarially train on the normalized inputs w.r.t. the rest of the relation $\calR\backslash \calR'$. Let $\calN_{R'}$ denote the normalizer for $\calR'$. Formally, the learner solves
\begin{equation}
\label{eq:combinedobj}
\min\limits_{f\in\calH}\max\limits_{A(\cdot)}\sum_{(\bfx,y)\in D}\ell\left(f,A\left(\calN_{\calR'}(\bfx)\right), y\right),
\end{equation}
during training to obtain a minimizer $f^*$, and predicts $f^*(\calN_{\calR'}(\bfx))$ at test-time. 
The classifier $f^*$ will be robust to $\calR'$-relational adversary, and the unified framework preserves the optimal robust accuracy as shown in \autoref{thm:equiv}. The adversarial training now uses a simpler relation $\calR\backslash\calR'$ and thus can potentially be more computationally efficient.

Last, we show in Appendix~\ref{subsec:complexity} an interesting connection between the adversarial example and the normal form: the strongest adversarial example satisfies the requirement of \autoref{thm:robustcondition}, and thus can be used as the normal form. Therefore, in theory, \nandp is at least as efficient as the optimal adversarial training. 
In practice, the normalizer we use in our empirical evaluation are all more efficient than adversarial training.

\section{\effnorm\ -- A Practical Normalizer}
\label{sec:effnorm}

The definition of normal form in \autoref{sec:master-normalizer} suggests a master algorithm of normalization: for each test input $\bfx$, we first find the weakly connect component (WCC) that contains the $\bfx$ in the relational graph, and returns a deterministic node in the WCC as the normal form. However, this algorithm is computationally demanding. The time complexity to identify the WCC requires enumeration over all possible variants to $\bfx$.
The size of the WCC can grow exponentially with the number of transformations. 

In response to the computation challenge of finding the exact normal form, we propose \effnorm\ -- a heuristic algorithm returning a code variant that is syntactically close to the exact normal form if not the same. In this section, we 
\begin{enumerate}[leftmargin=*]
    \item
    introduce the algorithmic description of \effnorm,
    \item 
    show how the graph search problem of normal form can be converted to an optimization problem given a similarity metric between code pieces,
    \item
    analyze the time complexity of \effnorm with respect to a special class of transformation named monotonically-decreasing transformation, and finally
    \item
    describe the implementation of our \effnorm\ for the authorship attribution task with actual running normalization examples.
\end{enumerate}

\subsection{Overview of \effnorm}
As shown in Algorithm~\ref{alg:effnorm}, \effnorm\ takes two input arguments: the test input $\bfx$ to be normalized, and a set of transformations $\calT_{\calN}$ that will be applied to $\bfx$ during normalization. Each element $T:\calX\to\calX \in \calT_{\calN}$ is a function that maps a code piece to a semantically-equivalent variant.

In each iteration, \effnorm finds a transformation in $\calT_{\calN}$ applicable to the current version pf code $\bfx_\calN$, and update $\bfx_{\calN}$ to the transformed version. The procedure ends when no transformation rule in $\calT_{\calN}$ can be applied to $\bfx_\calN$, and \effnorm\ returns the final $\bfx_{\calN}$ as the approximate normal form. 

The choice of $\calT_{\calN}$ is the most critical factor to the performance of \effnorm. In particular, our choice of $\calT_{\calN}$ needs to (1) guarantee that the normalization procedure eventually terminates, and (2) transform $\bfx_\calN$ towards the exact normal form. In the following sections, we will show how normalization can be formulated as an optimization problem, and with proper choice of the normal form and $\calT_{\calN}$, \effnorm becomes a heuristic algorithm of solving the optimization problem.
\begin{algorithm}
\caption{\effnorm$\left(\bfx, \calT_{\calN}\right)$}
{\bf input} $\bfx$: the test input; $\calT_{\calN}$: a set of transformations to be applied.
$\bfx_\calN = \bfx$\\
\While{$\exists T\in \calT_{\calN}$ s.t. $T\left(\bfx_\calN\right) \neq \bfx_\calN$}
{
  $\bfx_\calN = T\left(\bfx_\calN\right)$
}
\Return{$\bfx_\calN$}
\label{alg:effnorm}
\end{algorithm}

\subsection{Normalization as an Optimization Problem}

Recall that for a test input $\bfx$, we use $C_{\bfx}$ to denote the weakly connected component (WCC) that contains $\bfx$ in the relational graph and $\calN(\bfx)$ to be the exact normal form of $\bfx$. By \autoref{def:normal}, $\calN(\bfx)$ is also the normal form for all $\bfx'\in C_{\bfx}$.

Now, suppose we have a metric $d: \calX\to\calX \to [0,+\infty)$ that represents the similarity between two inputs $\bfx$ and $\bfx'$, then the challenge of finding the exact normal form can be formulated as the following minimization problem
\begin{equation}
\label{eq:opt1}
    \min_{\bfx'\in C_{\bfx}} d(\bfx', \calN(\bfx)).
\end{equation}

Moreover, suppose we have a function $\lambda: \calX\to\calR_0^+$ that maps each code piece to a unique real score, which essentially creates an order between possible inputs.  
Recall that any deterministic node in $C_{\bfx}$ can be the normal form for inputs in $C_{\bfx}$. We can define $$\calN(\bfx) \triangleq \arg\min_{\bfx'\in C_{\bfx}}\lambda(\bfx'),$$ i.e., the node in $C_{\bfx}$ with the smallest order. Then, the optimization problem in \autoref{eq:opt1} can be simplified to
\begin{equation}
\label{eq:opt2}
    \min_{\bfx'\in C_{\bfx}}\lambda(\bfx')
\end{equation}
by setting the metric $d(\bfx, \bfx') \triangleq |\lambda(\bfx) - \lambda(\bfx')|$.

\subsection{\effnorm with Monotonically Decreasing Transformation}
\label{sec:effnormperf}
The problem formulation in \autoref{eq:opt2} suggests a natural heuristic algorithm in a descent-fashion: we iteratively apply transformations that lower the objective function in \autoref{eq:opt2} until reaching a minimum. 

\begin{definition}[Strictly and strongly monotonically decreasing transformations]
A transformation rule $T:\calX\to\calX$ is {\bf \emph{strictly monotonically decreasing}} w.r.t. an order $\lambda$ \emph{iff} 
$$T(\bfx) \neq \bfx \implies \lambda(T(\bfx))<\lambda(\bfx)$$
for all $\bfx \in \calX$. Moreover, $t$ is called {\bf \emph{$\eps$-strongly decreasing}} \emph{iff} $$T(\bfx) \neq \bfx \implies \lambda(T(\bfx))<\lambda(\bfx)-\eps$$
for some constant $\eps>0$ and for all $\bfx \in \calX$.
\end{definition}

By choosing transformation rules in $\calT_{\calN}$ to be monotonically-decreasing w.r.t. a predefined order, the loop in \effnorm will always reduce the value of $\lambda(\bfx_{\calN})$ before termination, i.e., making $\bfx_{\calN}$ closer to the exact normal form $\calN(\bfx)$. The procedure is also guaranteed to terminate for $\eps$-strongly decreasing transformations because $\lambda$ is lower-bounded.
\begin{theorem}[Time-complexity of \effnorm]
\label{thm:timecomplexity}
Let $s = \lambda(\bfx)$ denote the score of input $\bfx$ under metric $\lambda$. Then for $\calT_{\calN}$ that contains only $\eps$-strongly decreasing transformations, \effnorm terminates within $s/\eps$ iterations.
\end{theorem}

Notice that in order to respect the boundary condition in Equation~\ref{eq:opt2} , $\calT_{\calN}$ should not transform the code outside $C_{\bfx}$. In practice, this boundary condition can be guaranteed by setting $\calT_{\calN}$ to be a subset of the assumed attacker's transformation rules.

\subsection{\effnorm for Source Code Attribution Attack} 
 We now implement \effnorm for the source code attribution attack in~\cite{USENIX-authorthisp}. In this section, we first describe the attack transformers considered by our \effnorm and explain the rationale of our choice. Next, we describe our choice of the order function $\lambda$ and the normalizing transformer set $\calT_{\calN}$. Last, we show a running example of \effnorm in action to illustrate the desired outcome. 
 
\begin{center}
\begin{table*}
\small
\caption{List of code transformations and their corresponding normalized transformation}
\centering
\begin{tabular}{|p{2cm}|p{1.5cm}|p{5.4cm}|p{2.5cm}|p{2.5cm}|}
\hline
\textbf{Transformer}           
& \textbf{Family} 
& \textbf{Transformer Description} 
& \textbf{Syntactic Feature of Interests (SFoI)} 
& \textbf{Normalizing Action in $\calT_\calN$}\\ 
\hline
For statement transformer          
& \multirow{3}{*}{Control}
& Replaces a for-statement by an equivalent while-statement             
& \multirow{2}{*}{Number of for loops}
& \multirow{2}{2.5cm}{Transform for loops to while}
\\ 
\cline{1-1} \cline{3-3}
While statement transformer        
&                                 
& Replaces a while-statement by an equivalent for-statement
&
&                                                                               \\ \cline{1-1} \cline{3-5} 
If statement transformer           
&                                 
& Split the condition of a single if-statement at logical operands (e.g., \&\& or ||) to create a cascade or a sequence of two if-statements depending on the logical operand          
& Number of logical predicate in an if-statement
& Split all if-statement so that each statement only has one logical predicate. \\ \hline
Array transformer                  
& \multirow{7}{*}{Declaration}    
& Converts a static or dynamically allocated array into a C++ vector object
& Number of array that can potentially be converted to C++ vector object
& Convert ALL static or dynamically allocated array into a C++ vector object.   \\ \cline{1-1} \cline{3-5} 
String transformer                 
&                                 
& Array option: Converts a char array (C-style string) into a C++ string object. The transformer adapts all usages in the respective scope, for instance, it replaces all calls to strlen by calling the instance methods size. String option: Converts a C++ string object into a char array (C-style string). The transformer adapts all usages in the respective scope, for instance, it deletes all calls to c str(). 
& Number of char array
& Convert all string to C++ string object.                                      \\ \cline{1-1} \cline{3-5} 
Integral type transformer          
&                                 
& Promotes integral types (char, short, int, long, long long) to the next higher type, e.g., int is replaced by long.        
& Number of integral type declaration lower than long long
& Replace all integral type with long long.                                     \\ \cline{1-1} \cline{3-5} 
Floating-point type transformer    
&                                 
& Converts float to double as next higher type.
& Number of float type declaration lower than double
& Convert all float to double.                                                  \\ \cline{1-1} \cline{3-5} 
Boolean transformer                
&                                 
& Bool option: Converts true or false by an integer representation to exploit the implicit casting. Int option: Converts an integer type into a boolean type if the integer is used as boolean value only
& Number of Boolean values
& Use int representation in all cases                                       \\ \cline{1-1} \cline{3-5} 
Init-Decl transformer              
&                                 
& Move into option: Moves a declaration for a control statement if defined outside into the control statement. For instance, int i; ...; for(i = 0; i \textless N; i++) becomes for(int i = 0; i \textless N; i++). Move out option: Moves the declaration of a control statement’s initialization variable out of the control statement.
& Number of looping variable declared outside the control statement
& Use move-in option for all scenarios.                                        \\ \cline{1-1} \cline{3-5} 
Typedef transformer                
&                                 
& Convert option: Convert a type from source file to a new type via typedef, and adapt all locations where the new type can be used. Delete option: Deletes a type definition (typedef) and replace all usages by the original data type.
& Number of user-defined types
& \multirow{1}{2.5cm}{Apply the delete option to all typedef.}                                       \\ 
\hline
\end{tabular}
\label{table:transformer}
\end{table*}
\end{center}

\begin{center}

\begin{table*}
\small
\caption*{Table 2: Continued}
\centering
\begin{tabular}{|p{2cm}|p{1.5cm}|p{5.4cm}|p{2.5cm}|p{2.5cm}|}
\hline
\textbf{Transformer}           
& \textbf{Family} 
& \textbf{Transformer Description} 
& \textbf{Syntactic Feature of Interests (SFoI)} 
& \textbf{Normalizing Action in $\calT_\calN$}\\ 
\hline
Include-typedef transformer        
& Template                        
& Inserts a type using typedef, and updates all locations where the new type can be used. Defaults are extracted from the 2016 Code Jam Competition.                                
& Number of user-defined types
& \multirow{1}{2.5cm}{Apply the delete option to all typedef.}                      
\\ \hline
Unused code transformer            
& Declaration                     
& Function option: Removes functions that are never called. Variable option: Removes global variables that are never used.              
& \multirow{1}{2.7cm}{Number of unused variables and functions}
& \multirow{1}{2.7cm}{Remove unused variables and functions.}                       \\ \cline{1-3}
\hline
Input interface transformer Output 
& \multirow{5}{*}{API}            
& Stdin option: Instead of reading the input from a file (e.g. by using the API ifstream or freopen), the input to the program is read from stdin directly (e.g. cin or scanf). File option: Instead of reading the input from stdin, the input is retrieved from a file.
& \multirow{2}{*}{Number of file I/O}
& Use stdin always.                         \\ 
\cline{1-1} \cline{3-3}
\cline{5-5}
Output interface transformer       
&                                 
& Stdout option: Instead of printing the output to a file (e.g. by ofstream or freopen), the output is written directly to stdout (e.g. cout or printf). File option: Instead of writing the output directly to stdout, the output is written to a file.
& 
& Use stdout always.                        \\ \cline{1-1} \cline{3-5} 
Input API transformer     
&                                 
& C++-Style option: Substitutes C APIs used for reading input (e.g., scanf) by C++ APIs (e.g., usage of cin). C-Style option: Substitutes C++ APIs used for reading input (e.g., usage of cin) by C APIs (e.g., scanf). & \multirow{2}{2.7cm}{Number of C-style I/O API}
& Use C++ API always                                                            \\ \cline{1-1} \cline{3-3} \cline{5-5} 
Output API transformer             
&
& C++-Style option: Substitutes C APIs used for writing output (e.g., printf) by C++ APIs (e.g., usage of cout). C-Style option: Substitutes C++ APIs used for writing output (e.g., usage cout) by C APIs (e.g., printf). 
&
& Use C++ Style always.                  
\\ \cline{1-1} \cline{3-5} 
Sync-with-stdio transformer        
&                                 
& Enable or remove the synchronization of C++ streams and C streams if possible         
& Number of potential synchronization sites
& Enable all synchronization.         
\\ \hline
Compound statement transformer     
& \multirow{2}{*}{Misc}  
& Insert option: Adds a compound statement (\{...\}). The transformer adds a new compound statement to a control statement (if, while, etc.) given their body is not already wrapped in a compound statement. Delete option: Deletes a compound statement (\{...\}). The transformer deletes compound statements that have no effect, i.e., compound statements containing only a single statement.            
& Number of compound statements
& Apply the delete option to all locations. \\ \cline{1-1} \cline{3-5} 
Return statement transformer       
&                                 
& Adds a return statement. The transformer adds a return statement to the main function to explicitly return 0 (meaning success). Note that main is a non-void function and is required to return an exit code. If the execution reaches the end of main without encountering a return statement, zero is returned implicitly.
& Number of implicit return sites
& Adds the return statement to all applicable locations.
\\ \hline
\end{tabular}
\end{table*}
\end{center}

\subsubsection{Attack Transformations to be Normalized}
The evasion and impersonation attack in \cite{USENIX-authorthisp} uses a total of 42 transformation options. These transformations can be divided into two categories based on the information needed. The \emph{general} transformations, such as changes in control statement, can be applied to any code without prior knowledge of the target author. In contrast, the \emph{template-based} transformations, which aim to mimic the target author's function/variable naming and type-def habits, will require samples of the target author's code pieces. \effnorm considers 35 of the the 42 transformations as listed in \autoref{table:transformer}.~\footnote{We follow the counting method in \cite{USENIX-authorthisp} and count by the number of options instead of number of transformers. For example, the input interface transformer has two options -- stdin and file; therefore, we count the input interface transformer as having two transformation options.} 
The selection covers most transformations including API usage, variable declaration, I/O style and control statement. The transformations are also common across attack papers~\cite{USENIX-authorthisp,SCAD-TIFS}. The only general transformation \emph{not} considered by \effnorm is function in-lining, which removes declared functions in the code and move the commands in-line to the caller function. While this may be done for simple online judge submissions, the boundary of usage of a function is in general hard to track in large code projects. In-lining function in one script may raise problems in other scripts that import and call the function. We do not normalize the transformation as it is hardly semantics-preserving~\cite{inlining}.
The template-based attack transformations \emph{not} considered by our \effnorm is function/variable renaming. The function/variable transformer peeks into the codes written by the target of imitation, and then rename the function and variables to match the target author's habit. This attack action requires extensive knowledge to the target author's coding habit. In the extreme case, the attacker may be more familiar to the target author than the learner, and thus makes the evasion inevitable. We do not normalize this transformation as it can significantly obscure the ground truth.

\begin{figure*}[t]
\centering
\includegraphics[width=0.8\textwidth]{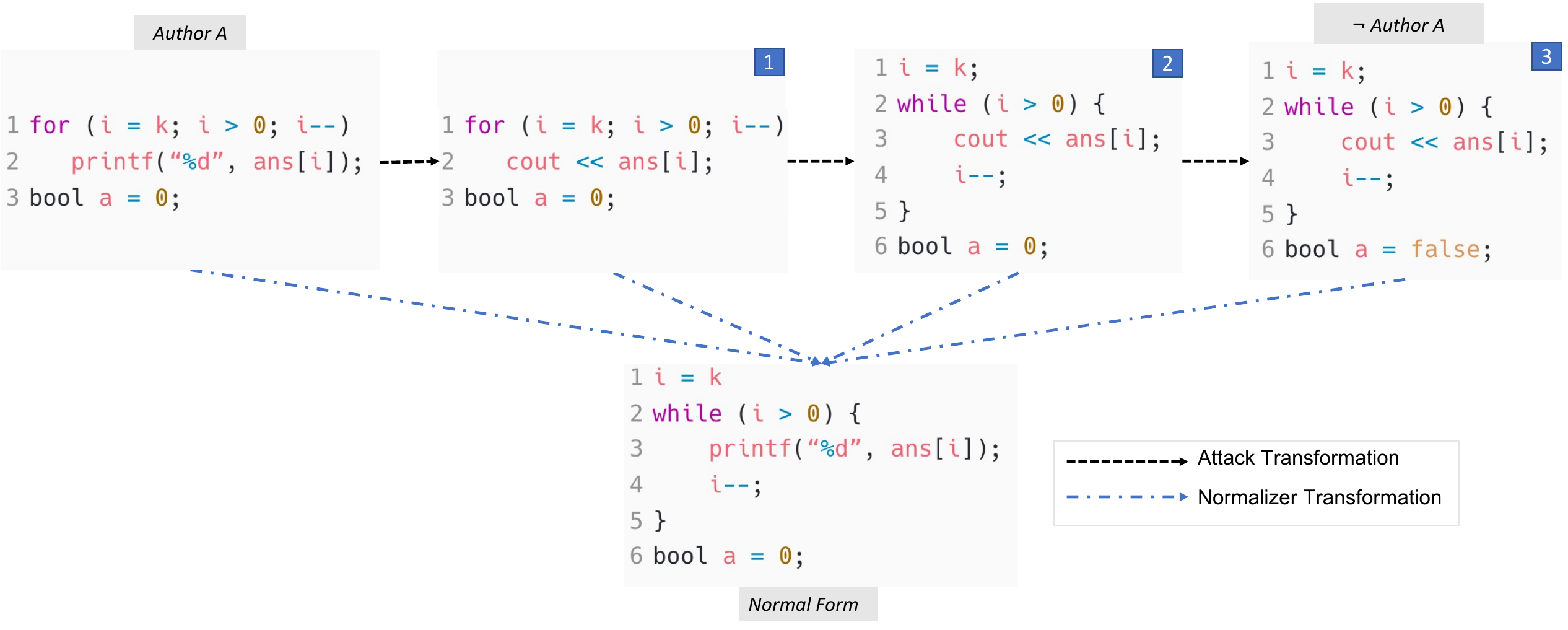}
\caption{Normalizer in action.}

\label{fig:normalizer}
\end{figure*}

\subsubsection{Order Function $\lambda$ \& Normalizing Transformer $\calT_{\calN}$}

Motivated by the theoretical findings of the performance of \effnorm in \autoref{sec:effnormperf}, we identify the Syntactic Features of Interests (SFoI) corresponding to the transformers as shown in \autoref{table:transformer}. Each SFoI's value can be changed by one or multiple transformers.
For example, we use the number of for loops in the code as one SFoI, and its value can be changed by two transformers -- the \emph{for-to-while} and the \emph{while-to-for} transformer over the control flow. 

We take the order function $\lambda$ as the \emph{sum} of values of the SFoIs, and the normal form of a code piece is the variant that has the smallest sum of SFoIs. For example, if the SFoIs are 1) number of \emph{for} loops and 2) number of \emph{printf} statements, then the normal form will be the variant that has the least number of \emph{for} loops and \emph{printf} statements in total. In \autoref{table:transformer}, we identify a total of 16 SFoIs and use the sum of all of them as our order function $\lambda$.

We select a subset of the attack transformations as the set of normalizing transformations $\calT_{\calN}$ so as to respect the boundary condition in \autoref{eq:opt2}. Our $\calT_{\calN}$ contains and only contains the attack transformations that strictly decrease the value of SFoIs. For example, when the number of \emph{for} loops is used as an SFoI, we keep the \emph{for-to-while} transformer in $\calT_{\calN}$ and discard the \emph{while-to-for}. The right-most column of \autoref{table:transformer} shows all the transformations we keep in $\calT_{\calN}$. Notice that the SFoIs all take non-negative integer values, and the normalizing transformations are all $1$-strongly decreasing. Therefore, as a result of \autoref{thm:timecomplexity}, the number of iterations in \effnorm for an input $\bfx$ is bounded by the value of $\lambda(\bfx)$. 

\subsubsection{\effnorm in Action}
\autoref{fig:normalizer} shows an example of \effnorm in action. The left-most box contains an code snippet originally written by Author A. The subsequent code boxes in the top row illustrate a sequence of transformations applied to the original code. The attacker first converts the C-style \emph{printf} statement to the C++ style \emph{cout} statement, then changes the \emph{for} loop to a \emph{while} loop, and eventually changes the value of a boolean variable from 0 to False. The final variant in right-most box is a successful adversarial example that misled the model to predict a different author.

In this code example, three syntactical features of interests are involved: 1) the number of \emph{for} loops, 2) the number of C-style I/O statements and 3) the number of Boolean values that can be cast into integers. \effnorm\ applies the normalizing actions in a iterative manner, reducing the number of \emph{for} loops, C-style IO statements and Boolean values until no more action is applicable. All four code pieces -- the original input, the final adversarial example and the two intermediate variants -- will be normalized into the same normal form as depicted in the bottom box in \autoref{fig:normalizer}.
\section{Evaluation} \label{sec:exp}
Having built the \nandp framework and an efficient normalizer for the transformations in~\cite{USENIX-authorthisp}, we now proceed to evaluate the performance of the \nandp models in the presence of the source code attribution attack. In particular, we investigate the following questions:
\begin{enumerate}
    \item
     Does \nandp improve robust accuracy compared to the unprotected models and the adversarial trained counterpart in practice?
    \item
    To what extent does \effnorm reduce the difference between the original and the adversarial input? 
    \item
    How does the computational overhead of \nandp\ compare to that of adversarial training?
\end{enumerate}
We evaluate the performance of our \nandp\ framework over the same data set and threat model in~\cite{USENIX-authorthisp}. The results corroborates our theoretical findings: \nandp achieves significantly higher robust accuracy compared to the baseline models; it also overcomes the pain-points of adversarial training as it significantly shortens the training time with minimum change to the existing learning pipeline.

\subsection{Dataset}
We use the dataset provided by Quiring et al ~\cite{USENIX-authorthisp}.~\footnote{https://github.com/EQuiw/code-imitator/tree/master/data/dataset\_2017} The data set is collected from Google Code Jam~\cite{googleJAM}, a coding platform on which individual programmers compete to solve coding challenges. It consists of 1,632 files of C++ code from 204 authors solving the same 8 programming challenges of the competition. 
We follow the same train-test data splits in~\cite{USENIX-authorthisp}. We create 8 different data splits. Each split uses the codes from one challenge as the test set and the codes from the rest seven challenges for training. We run the experiments over all 8 splits and report the average results.

\subsection{Implementation and Infrastructure}
We implement source code \textit{normalizer} on top of Clang~\cite{clang}, an open-source C/C++ frontend for the LLVM compiler framework. For fair comparison of running time, we run all experiment series on an Amazon EC2 c5.18xlarge instance with 72 cores and 144GB memory. We train the model using the \emph{sklearn} and \emph{keras} APIs with \emph{tensorflow} backend.

\begin{table*}[h!]
\centering
\large
\caption{Model accuracy(\%) under attack. The leader for each series is highlighted in bold.}
\renewcommand{\arraystretch}{1.25}
\resizebox{0.8\textwidth}{!}{%
\begin{tabular}{l|c|c|c|c|c} 
& \nandp-RF 
& \nandp-LSTM 
& Adv-LSTM
& Vanilla-RF (Calistan et al.~\cite{rf-caliskan})
& Vanilla-LSTM (Abuhamad et al.~\cite{rnn-abuhamad})
\\
\hline
Clean
&76.1 &78.7  &76.7  &\textbf{90.4}  &88.4 \\
Matched 
&72.3 &\textbf{73.7}  &30.8  &13.2  & 21.1 \\
Full
&70.2 &\textbf{71.5} &25.1 &0.8 &0.9 \\
Adaptive
&37.3 &\textbf{49.9} &26.2 &10.2 &21.1\\
\end{tabular}}
\label{table:acc}
\end{table*}

\subsection{Baseline}
\paragraph{Learning Models.} We consider both the random forest (RF) model and the recurrent neural net model with LSTM units attacked in~\cite{USENIX-authorthisp} for baseline performance. Hereinafter, we name them 
{\bf Vanilla-RF} and {\bf Vanilla-LSTM}. Our \nandp\ framework uses the same model structure, parameters and training procedure as Quiring et al.~\cite{USENIX-authorthisp} except that \nandp works with normalized data instead of their original form. Hereinafter, we call the normalized models {\bf \nandp-RF} and {\bf \nandp-LSTM}. For the adversarial training baseline, we train the lone compatible model {\bf Vanilla-LSTM} into {\bf Adv-LSTM}. 
We note that the standard adversarial training, which trains the model from scratch and creates adversarial examples for each gradient descent step with respect to the current model, is too computationally inefficient for the source attribution attack. In existing adversarial training literature, the $\ell_p$-norm based adversarial examples are created directly in the feature-space using gradient ascent; the attack can be readily computed in GPU in a similar manner as model updates.
The code transformations, however, are performed in the problem-space; the MCTS computation is CPU-intensive and thus takes much longer. In addition, the validity check of adversarial examples further increases the computational load. We make the following adjustments to speed up the adversarial training process. 
First, we use \textbf{Vanilla-LSTM} as the initial model and fine-tune it using adversarial inputs. The models show improved robust accuracy (\textasciitilde60\%) to the attacks in the training loops after 10 epochs. Second, instead of generating adversarial examples at every training step, we generate the adversarial training inputs for all .cpp files with respect to the model at beginning of an epoch. This change allows us to generate the adversarial training inputs in parallel. To further speed up the adversarial training procedure, we also reduce the max-depth from 25 to 5 as well as the number of random play-outs at each node from 50 to 10 in the Monte-Carlo tree-search. With these modification, we finally manage to finish adversarial training within a month.

\paragraph{Attack Methods.} We consider three attack modes. All three attacks use the Monte-Carlo tree search method in~\cite{USENIX-authorthisp} with the same parameters including number of iterations and roll-out per iteration. The difference among the three is that the {\bf Matched} attack uses the 35 transformations normalized by our \effnorm; the {\bf Full} attack uses all 42 transformations in~\cite{USENIX-authorthisp}; and the {\bf Adaptive} attack uses the 7 transformations not considered in \effnorm. We assume an adaptive attacker who has full knowledge of our \nandp pipeline. By initializing the input to the normal form and attacking using the above-mentioned 7 transformations, the attacker can generate adversarial examples invariant to the transformations in \effnorm.

\subsection{Accuracy Results and Discussion}

\noindent \mypara{\nandp v.s. Vanilla Models.} \autoref{table:acc} shows the test accuracy of all baselines against adversarial examples and clean inputs. 
\textbf{\nandp-LSTM} has the highest accuracy against all three attacks followed by \textbf{\nandp-RF}. Compared to the corresponding vanilla models with the same model structures, the \nandp models achieves higher accuracy by a wide margin. The accuracy increases by more than 50\% for the matched attack and by around 70\% for the full attack. The results show that \nandp is highly effective when the attacker uses transformations already considered by the normalizer. 

Intriguingly, even under the adaptive attack in which the adversarial example remains the same after normalization, \nandp still achieves nontrivial accuracy --- 37.3\% for \textbf{\nandp-RF} and 49.9\% for \textbf{\nandp-LSTM}, 27\%/28\% higher than {\bf Vanilla-RF}/{\bf Vanilla-LSTM}. By forcing the adaptive attacker to use a smaller set of transformations, \nandp effectively reduces the attack surface and thus still enhances the accuracy. The extent of the improvement depends on the transformations used in the adaptive attack. The improvement over the vanilla model is also partly due to the nature of the adaptive attack: the initial state of the Monte-Carlo tree search is set to the normal form, which potentially has already wiped out some transformations by the normalizer.

For clean inputs without undergoing adversarial transformations, the \nandp models has lower accuracy compared to the vanilla models due to the inevitable robustness-accuracy trade-off analyzed in \autoref{sec:nandp}. However, the difference (<10\% for LSTM and <15\% for RF) is much smaller compared to the accuracy gain in the adversarial setting. The trade-off will be worthwhile when a sizable portion of the inputs come from adversarial sources.

\vspace{0.25\baselineskip}
\noindent \mypara{\nandp v.s. Adversarial Training.} In addition to the improvement over the vanilla models, \nandp models also consistently outperforms the adversarially trained counterparts across all attacks by a significant margin (>40\% for matched, >45\% for full and up to 23\% for adaptive). 

The performance of adversarial training is heavily affected by the strength of the attack used in the training loop: a model adversarially trained with a weak attack in the loop may succumb to a strong attack in test time. Recall that our adversarial training procedure uses an attack with reduced search parameters in order to have reasonable training time. {\bf Adv-LSTM} is able to achieve \textasciitilde 60\% accuracy against the full attack in the training loop. However, the accuracy drops significantly to 25.1\% against the full-strength MCTS attack in the actual test. In security applications, the attacker is often assumed to have more computation power than the defender and may use attack algorithms unknown to the defender in which case \nandp is likely to show more consistent performance. 

\begin{table}[h!]
\centering
\caption{Average $\ell_2$-norm of the original input $\bfx$, $\ell_2$-norm of the adversarial perturbation $\Delta\bfx$ and the perturbation ratio $\|\Delta\bfx\|/\|\bfx\|$ in feature space under full attack.}
\renewcommand{\arraystretch}{1.25}
\resizebox{0.45\textwidth}{!}{%
\begin{tabular}{c|c|c|c|c} 

& \nandp-RF 
& \nandp-LSTM 
& Vanilla-RF
& Vanilla-LSTM
\\
\hline
$\|\bfx\|$ &56.8 & 25.7 & 48.5 &25.6  \\
$\|\Delta\bfx\|$ &3.8 &3.9  & 10.1 &16.0 \\
$\|\Delta\bfx\|/\|\bfx\|$ &6.8\%  &14.3\% &21.4\% &61.4\% \\
\end{tabular}
}
\label{table:feature}
\end{table}

\subsection{Performance of \effnorm}
Our heuristic normalizer \effnorm is proven to be effective at reducing the syntax difference between the original inputs and their adversarial examples. \autoref{table:feature} shows the average $\ell_2$ norm of the original input, the norm of the adversarial perturbation and their ratio in feature space under full attack. Without \effnorm, the perturbation ratio is 21.4\% for vanilla-RF and as much as 61.4\% for vanilla-LSTM. The adversarial transformations have drastically changed the feature representation of the input. However, after passing both the clean and adversarial example through \effnorm, the distortion ratio is reduced to merely 6.8\% for RF and 14.3\% for LSTM. The reduction in distortion suggests that the attacker can manipulate the model input to a much less degree, which results in higher accuracy under attack. 

The effectiveness of \effnorm can be further validated by the performance of our \nandp framework. Recall that the matched attack uses the same transformation considered by our normalizer. In the ideal case of a perfect normalizer, \nandp models should have the same accuracy over the clean inputs and the adversarial examples created by the matched attack. \effnorm is close to the ideal case as the discrepancy between the clean accuracy and the accuracy against the matched attack is only 3.8\% for RF and 4.0\% for LSTM.

\subsection{Running Time}
We show the training time of the learning frameworks in \autoref{table:time}. The vanilla models take less than 12 hours to train on our infrastructure. 
Compared to the vanilla learning pipeline, \nandp incurs an overhead in normalization. It takes less than 12 hours to normalize the entire data set, which is in the same order of the vanilla training time. In contrast, adversarial training requires much longer training time. Even after we reduce the search parameters in the attack, adversarial training still takes more than 20 days to finish on the same infrastructure, which is 40x more than the vanilla training time. \nandp shows a clear advantage in terms of running time.

\begin{table}[h!]
\centering
\caption{Training time of different pipelines.}
\renewcommand{\arraystretch}{1.25}
\resizebox{0.4\textwidth}{!}{%
\begin{tabular}{c|c|c} 
 Vanilla 
& \nandp Overhead
& Adv. Training Overhead
\\
\hline
< 12 hrs &< 12 hrs  &>  20 days\\
\end{tabular}
}
\label{table:time}
\end{table}

\section{Related Work}
\label{sec:related_work}

Test-time attacks using adversarial examples have been extensively studied in the past several years. Research has shown ML models are vulnerable to such attack in a variety of application domains~\cite{moosavi2016deepfool, chen2017zoo, papernot2017practical, eykholt2018robust, ebrahimi2018hotflip, qin2019imperceptible, yang2020greedy, yefet2020adversarial} including system security and program analysis where reliable defense is absolutely essential. In this paper, we focus on the authorship attribution task that has been compromised by \cite{USENIX-authorthisp,SCAD-TIFS}. Similarly,
\cite{grosse2017adversarial} and~\cite{al2018adversarial} successfully evade API/library usage based malware detectors by adding redundant API calls; \cite{rosenberg2018generic},~\cite{hu2018black}, and~\cite{rosenberg2019defense} successfully attack running-time behavior based detectors by adding redundant execution traces. In contrast to the early attack that directly perturbs the feature representation, \cite{pierazzi2020problemspace} emphasizes the importance of generating adversarial examples in the \emph{problem-space} and creates realistic attack instances using automated software transplantation.
Subsequent work~\cite{USENIX-authorthisp, SCAD-TIFS, ramakrishnan2020semantic} has shown effective attacks using semantic-preserving transformation in the problem-space.

Both the machine learning and the security communities have put in vast amount of effort to enhance model robustness v.s. test-time adversarial manipulation. Adversarial training~\cite{madry2018towards, wong2020fast, shafahi2019adversarial} is one of the most successful principled defense mechanism for deep learning models. For non-parametric models, \cite{wang2018analyzing, yang2019adversarial} proposes data-pruning based defenses with guarantees. However, these defenses mainly consider adversarial perturbations with small $\ell_p$ norm. Problem-space transformations, however, can cause large perturbation in the feature space that exceeds the limit of protection of these defenses.
In response to such limitation, recent work starts to develop defense mechanisms for transformation based adversaries~\cite{yang2019invariance,grosse2017adversarial, al2018adversarial, rosenberg2019defense,Incer2018ARM, kouzemtchenko2018defending}. Yang et al.~\cite{yang2019invariance} adds invariance-induced regularizers to the training process against a specific spatial transformation attack in image classification to achieve high model accuracy. In contrast, our work considers a general adversary based on logic relations and also emphsizes provable robustness guarantee. Hendrycks et al. \cite{hendrycks2019using} uses the semi-supervised learning technique which transforms the input in training and let the model learn to both detect the transformation and predict the class label. However, the model only detects single or a short sequence of transformations in the training process and thus cannot handle long sequence of transformations in test-time.
\cite{grosse2017adversarial, al2018adversarial, rosenberg2019defense, zhang2021challenging} improve robustness via adversarial training with transformed inputs; we show such approach is hard to optimize.~\cite{Incer2018ARM, kouzemtchenko2018defending} enforce monotonicity over model outputs so that the addition of feature values always increase the maliciousness score. These approaches are limited to guarding against the addition attacks, thus lacks generality. 

Normalization is a technique to reduce the number of syntactically distinct instances.
First introduced to network security in the early 2000s in the context of intrusion detection systems~\cite{handleypaxsonkreibich-sec01},
it was later applied to 
malware detection~\cite{christodorescujhakinder-jcv07, Coogan:2011:DVS:2046707.2046739, Bichsel:2016:SDA:2976749.2978422, Salem:2016:MRO:3015135.3015136, Baumann:2017:ATA:3099012.3099020}. Our work addresses the open question whether normalization is useful for ML under relational adversary by investigating its impact on both model robustness and accuracy.

\section{Conclusion and Future Work}
\label{sec:conclusion}
In this work, we set the first step towards robust learning against relational adversaries: we theoretically characterize the conditions for robustness and the sources of robustness-accuracy trade-off, and propose a provably robust learning framework. Motivated by the theoretical insights, we instantiate the framework for the application of authorship attribution. Our empirical evaluation shows that the normalize-and-predict approach can significantly enhance model robustness against attacks using semantic-preserving transformation at much computation overhead.

For future work, we see promising directions in both improving the normalizer and using the \nandp framework in broader applications. A normalizer for general purpose transformations over more types of codes and programming languages can help gauging the robustness of models in more applications with larger scale. \nandp could also be used in non-adversarial setting, e.g. removing biases in data sets and enhancing model interpretability. Last, we are eager to explore the synergy of \nandp with other defense mechanism such as adversarial training to harness the benefits from both worlds.

\bibliographystyle{plain}
\bibliography{robust-detection,mohannad}

\begin{thebibliography}{10}

\bibitem{DL-CAIS}
Mohammed Abuhamad, Tamer AbuHmed, Aziz Mohaisen, and DaeHun Nyang.
\newblock Large-scale and language-oblivious code authorship identification.
\newblock In {\em Proceedings of the 2018 ACM SIGSAC Conference on Computer and
  Communications Security}, CCS '18, page 101–114, New York, NY, USA, 2018.
  Association for Computing Machinery.

\bibitem{rnn-abuhamad}
Mohammed Abuhamad, Tamer AbuHmed, Aziz Mohaisen, and DaeHun Nyang.
\newblock Large-scale and language-oblivious code authorship identification.
\newblock In {\em Proceedings of the 2018 ACM SIGSAC Conference on Computer and
  Communications Security}, pages 101--114, 2018.

\bibitem{al2018adversarial}
Abdullah Al-Dujaili, Alex Huang, Erik Hemberg, and Una-May O’Reilly.
\newblock Adversarial deep learning for robust detection of binary encoded
  malware.
\newblock In {\em 2018 IEEE Security and Privacy Workshops (SPW)}, pages
  76--82. IEEE, 2018.

\bibitem{alsulami2017source}
Bander Alsulami, Edwin Dauber, Richard Harang, Spiros Mancoridis, and Rachel
  Greenstadt.
\newblock Source code authorship attribution using long short-term memory based
  networks.
\newblock In {\em European Symposium on Research in Computer Security}, pages
  65--82. Springer, 2017.

\bibitem{Baumann:2017:ATA:3099012.3099020}
Richard Baumann, Mykolai Protsenko, and Tilo M\"{u}ller.
\newblock Anti-proguard: Towards automated deobfuscation of android apps.
\newblock In {\em Proceedings of the 4th Workshop on Security in Highly
  Connected IT Systems}, SHCIS '17, pages 7--12, New York, NY, USA, 2017. ACM.

\bibitem{Bichsel:2016:SDA:2976749.2978422}
Benjamin Bichsel, Veselin Raychev, Petar Tsankov, and Martin Vechev.
\newblock Statistical deobfuscation of android applications.
\newblock In {\em Proceedings of the 2016 ACM SIGSAC Conference on Computer and
  Communications Security}, CCS '16, pages 343--355, New York, NY, USA, 2016.
  ACM.

\bibitem{De-Anonymizing}
Aylin Caliskan-Islam, Richard Harang, Andrew Liu, Arvind Narayanan, Clare Voss,
  Fabian Yamaguchi, and Rachel Greenstadt.
\newblock De-anonymizing programmers via code stylometry.
\newblock In {\em Proceedings of the 24th USENIX Conference on Security
  Symposium}, SEC'15, page 255–270, USA, 2015. USENIX Association.

\bibitem{rf-caliskan}
Aylin Caliskan-Islam, Richard Harang, Andrew Liu, Arvind Narayanan, Clare Voss,
  Fabian Yamaguchi, and Rachel Greenstadt.
\newblock De-anonymizing programmers via code stylometry.
\newblock In {\em 24th {USENIX} Security Symposium ({USENIX} Security 15)},
  pages 255--270, 2015.

\bibitem{chen2017zoo}
Pin-Yu Chen, Huan Zhang, Yash Sharma, Jinfeng Yi, and Cho-Jui Hsieh.
\newblock Zoo: Zeroth order optimization based black-box attacks to deep neural
  networks without training substitute models.
\newblock In {\em Proceedings of the 10th ACM Workshop on Artificial
  Intelligence and Security}, pages 15--26. ACM, 2017.

\bibitem{christodorescujhakinder-jcv07}
Mihai Christodorescu, Somesh Jha, Johannes Kinder, Stefan Katzenbeisser, and
  Helmut Veith.
\newblock Software transformations to improve malware detection.
\newblock {\em Journal in Computer Virology}, 3:253--265, 10 2007.

\bibitem{clang}
Clang: a c language family frontend for llvm.
\newblock \url{https://clang.llvm.org/}.
\newblock Accessed: 2021-10-24.

\bibitem{Coogan:2011:DVS:2046707.2046739}
Kevin Coogan, Gen Lu, and Saumya Debray.
\newblock Deobfuscation of virtualization-obfuscated software: A
  semantics-based approach.
\newblock In {\em Proceedings of the 18th ACM Conference on Computer and
  Communications Security}, CCS '11, pages 275--284, New York, NY, USA, 2011.
  ACM.

\bibitem{ebrahimi2018hotflip}
Javid Ebrahimi, Anyi Rao, Daniel Lowd, and Dejing Dou.
\newblock Hotflip: White-box adversarial examples for text classification.
\newblock In {\em Proceedings of the 56th Annual Meeting of the Association for
  Computational Linguistics (Volume 2: Short Papers)}, pages 31--36, 2018.

\bibitem{eykholt2018robust}
Kevin Eykholt, Ivan Evtimov, Earlence Fernandes, Bo~Li, Amir Rahmati, Chaowei
  Xiao, Atul Prakash, Tadayoshi Kohno, and Dawn Song.
\newblock Robust physical-world attacks on deep learning visual classification.
\newblock In {\em Proceedings of the IEEE Conference on Computer Vision and
  Pattern Recognition}, pages 1625--1634, 2018.

\bibitem{Frantzeskou2006}
Georgia Frantzeskou, Efstathios Stamatatos, Stefanos Gritzalis, and Sokratis
  Katsikas.
\newblock Effective identification of source code authors using byte-level
  information.
\newblock In {\em Proceedings of the 28th International Conference on Software
  Engineering}, ICSE '06, page 893–896, New York, NY, USA, 2006. Association
  for Computing Machinery.

\bibitem{googleJAM}
Google code jam.
\newblock \url{https://code.google.com/codejam/}.
\newblock Accessed: 2021-10-24.

\bibitem{grosse2017adversarial}
Kathrin Grosse, Nicolas Papernot, Praveen Manoharan, Michael Backes, and
  Patrick McDaniel.
\newblock Adversarial examples for malware detection.
\newblock In {\em European Symposium on Research in Computer Security}, pages
  62--79. Springer, 2017.

\bibitem{normalizer-2001}
Mark Handley, Vern Paxson, and Christian Kreibich.
\newblock Network intrusion detection: Evasion, traffic normalization, and
  {End-to-End} protocol semantics.
\newblock In {\em 10th USENIX Security Symposium (USENIX Security 01)},
  Washington, D.C., August 2001. USENIX Association.

\bibitem{handleypaxsonkreibich-sec01}
Mark Handley, Vern Paxson, and Christian Kreibich.
\newblock Network intrusion detection: Evasion, traffic normalization, and
  end-to-end protocol semantics.
\newblock In {\em Proceedings of the 10th Conference on USENIX Security
  Symposium - Volume 10}, SSYM'01, Berkeley, CA, USA, 2001. USENIX Association.

\bibitem{hendrycks2019using}
Dan Hendrycks, Mantas Mazeika, Saurav Kadavath, and Dawn Song.
\newblock Using self-supervised learning can improve model robustness and
  uncertainty.
\newblock {\em arXiv preprint arXiv:1906.12340}, 2019.

\bibitem{hu2018black}
Weiwei Hu and Ying Tan.
\newblock Black-box attacks against rnn based malware detection algorithms.
\newblock In {\em Workshops at the Thirty-Second AAAI Conference on Artificial
  Intelligence}, 2018.

\bibitem{Incer2018ARM}
Inigo Incer, Michael Theodorides, Sadia Afroz, and David Wagner.
\newblock Adversarially robust malware detection using monotonic
  classification.
\newblock In {\em the Fourth ACM International Workshop on Security and Privacy
  Analytics (IWSPA)}, Tempe, AZ, USA, Mar. 2018.

\bibitem{survey-Authorship-Attribution}
Vaibhavi Kalgutkar, Ratinder Kaur, Hugo Gonzalez, Natalia Stakhanova, and Alina
  Matyukhina.
\newblock Code authorship attribution: Methods and challenges.
\newblock {\em ACM Comput. Surv.}, 52(1), feb 2019.

\bibitem{reluplex}
G.~Katz, C.~Barrett, D.L. Dill, K.~Julian, and M.J. Kochenderfer.
\newblock Reluplex: An efficient smt solver for verifying deep neural networks.
\newblock In {\em International Conference on Computer Aided Verification},
  2017.

\bibitem{kouzemtchenko2018defending}
Alex Kouzemtchenko.
\newblock Defending malware classification networks against adversarial
  perturbations with non-negative weight restrictions.
\newblock {\em arXiv preprint arXiv:1806.09035}, 2018.

\bibitem{kozen:pspace}
Dexter Kozen.
\newblock Lower bounds for natural proof systems.
\newblock In {\em FOCS}, 1977.

\bibitem{SCAD-TIFS}
Qianjun Liu, Shouling Ji, Changchang Liu, and Chunming Wu.
\newblock A practical black-box attack on source code authorship identification
  classifiers.
\newblock {\em IEEE Transactions on Information Forensics and Security},
  16:3620--3633, 2021.

\bibitem{madry2018towards}
Aleksander Madry, Aleksandar Makelov, Ludwig Schmidt, Dimitris Tsipras, and
  Adrian Vladu.
\newblock Towards deep learning models resistant to adversarial attacks.
\newblock In {\em 6th International Conference on Learning Representations
  (ICLR)}, Vancouver, Canada, Apr. 2018.

\bibitem{moosavi2016deepfool}
Seyed-Mohsen Moosavi-Dezfooli, Alhussein Fawzi, and Pascal Frossard.
\newblock Deepfool: a simple and accurate method to fool deep neural networks.
\newblock In {\em Proceedings of the IEEE conference on computer vision and
  pattern recognition}, pages 2574--2582, 2016.

\bibitem{papernot2017practical}
Nicolas Papernot, Patrick McDaniel, Ian Goodfellow, Somesh Jha, Z~Berkay Celik,
  and Ananthram Swami.
\newblock Practical black-box attacks against machine learning.
\newblock In {\em Proceedings of the 2017 ACM on Asia conference on computer
  and communications security}, pages 506--519, 2017.

\bibitem{pierazzi2020problemspace}
Fabio Pierazzi, Feargus Pendlebury, Jacopo Cortellazzi, and Lorenzo Cavallaro.
\newblock Intriguing properties of adversarial {ML} attacks in the problem
  space.
\newblock In {\em 2020 IEEE Symposium on Security and Privacy (SP)}, pages
  1308--1325. IEEE Computer Society, 2020.

\bibitem{qin2019imperceptible}
Yao Qin, Nicholas Carlini, Garrison Cottrell, Ian Goodfellow, and Colin Raffel.
\newblock Imperceptible, robust, and targeted adversarial examples for
  automatic speech recognition.
\newblock In {\em International Conference on Machine Learning}, pages
  5231--5240, 2019.

\bibitem{USENIX-authorthisp}
Erwin Quiring, Alwin Maier, and Konrad Rieck.
\newblock Misleading authorship attribution of source code using adversarial
  learning.
\newblock In {\em Proceedings of the 28th USENIX Conference on Security
  Symposium}, SEC'19, page 479–496, USA, 2019. USENIX Association.

\bibitem{ramakrishnan2020semantic}
Goutham Ramakrishnan, Jordan Henkel, Zi~Wang, Aws Albarghouthi, Somesh Jha, and
  Thomas Reps.
\newblock Semantic robustness of models of source code, 2020.

\bibitem{rosenberg2019defense}
Ishai Rosenberg, Asaf Shabtai, Yuval Elovici, and Lior Rokach.
\newblock Defense methods against adversarial examples for recurrent neural
  networks.
\newblock {\em arXiv preprint arXiv:1901.09963}, 2019.

\bibitem{rosenberg2018generic}
Ishai Rosenberg, Asaf Shabtai, Lior Rokach, and Yuval Elovici.
\newblock Generic black-box end-to-end attack against state of the art api call
  based malware classifiers.
\newblock In {\em International Symposium on Research in Attacks, Intrusions,
  and Defenses}, pages 490--510. Springer, 2018.

\bibitem{Salem:2016:MRO:3015135.3015136}
Aleieldin Salem and Sebastian Banescu.
\newblock Metadata recovery from obfuscated programs using machine learning.
\newblock In {\em Proceedings of the 6th Workshop on Software Security,
  Protection, and Reverse Engineering}, SSPREW '16, pages 1:1--1:11, New York,
  NY, USA, 2016. ACM.

\bibitem{shafahi2019adversarial}
Ali Shafahi, Mahyar Najibi, Amin Ghiasi, Zheng Xu, John Dickerson, Christoph
  Studer, Larry~S Davis, Gavin Taylor, and Tom Goldstein.
\newblock Adversarial training for free!
\newblock {\em arXiv preprint arXiv:1904.12843}, 2019.

\bibitem{simko2018recognizing}
Lucy Simko, Luke Zettlemoyer, and Tadayoshi Kohno.
\newblock Recognizing and imitating programmer style: Adversaries in program
  authorship attribution.
\newblock {\em Proc. Priv. Enhancing Technol.}, 2018(1):127--144, 2018.

\bibitem{inlining}
Nicolas Stucki, Aggelos Biboudis, S\'{e}bastien Doeraene, and Martin Odersky.
\newblock {\em Semantics-Preserving Inlining for Metaprogramming}, page
  14–24.
\newblock Association for Computing Machinery, New York, NY, USA, 2020.

\bibitem{wang2018analyzing}
Yizhen Wang, Somesh Jha, and Kamalika Chaudhuri.
\newblock Analyzing the robustness of nearest neighbors to adversarial
  examples.
\newblock In {\em International Conference on Machine Learning}, pages
  5133--5142. PMLR, 2018.

\bibitem{wong2020fast}
Eric Wong, Leslie Rice, and J~Zico Kolter.
\newblock Fast is better than free: Revisiting adversarial training.
\newblock {\em arXiv preprint arXiv:2001.03994}, 2020.

\bibitem{yang2019invariance}
Fanny Yang, Zuowen Wang, and Christina Heinze-Deml.
\newblock Invariance-inducing regularization using worst-case transformations
  suffices to boost accuracy and spatial robustness.
\newblock In {\em Advances in Neural Information Processing Systems}, pages
  14757--14768, 2019.

\bibitem{yang2020greedy}
Puyudi Yang, Jianbo Chen, Cho-Jui Hsieh, Jane-Ling Wang, and Michael~I Jordan.
\newblock Greedy attack and gumbel attack: Generating adversarial examples for
  discrete data.
\newblock {\em Journal of Machine Learning Research}, 21(43):1--36, 2020.

\bibitem{yang2019adversarial}
Yao-Yuan Yang, Cyrus Rashtchian, Yizhen Wang, and Kamalika Chaudhuri.
\newblock Robustness for non-parametric classification: A generic attack and
  defense.
\newblock In {\em International Conference on Artificial Intelligence and
  Statistics}, pages 941--951. PMLR, 2020.

\bibitem{yefet2020adversarial}
Noam Yefet, Uri Alon, and Eran Yahav.
\newblock Adversarial examples for models of code.
\newblock {\em Proceedings of the ACM on Programming Languages},
  4(OOPSLA):1--30, 2020.

\bibitem{zhang2021challenging}
Weiwei Zhang, Shengjian Guo, Hongyu Zhang, Yulei Sui, Yinxing Xue, and Yun Xu.
\newblock Challenging machine learning-based clone detectors via
  semantic-preserving code transformations.
\newblock {\em arXiv preprint arXiv:2111.10793}, 2021.

\end{thebibliography}

\newpage
\appendix
\section{Proofs and Explanation for Theoretical Results}
\label{sec:appendix-proof}
In this section, we present the omitted proofs for theorems and observations due to page limit of the main body.

\subsection{Proof for Observation~\ref{thm:robustcondition}}
The if direction holds because any $\bfx \in C_i$ can only be transformed into $\bfz \in C_i$ by the maximal property of weakly connected component. The contrapositive of the only if direction holds because if $f(\bfx)\neq f(\bfz)$ for a pair $\bfx, \bfz \in C_i$, then there must exist two adjacent nodes $\bfz_1, \bfz_2$ on the path between $\bfx$ and $\bfz$ such that one can be transformed to another yet $f(\bfz_1) \neq f(\bfz_2)$.

\subsection{Proof for Observation~\ref{thm:tradeoff}}
\label{sec:tradeoffproof}
We first write the full version for Observation~\ref{thm:tradeoff} in the following claim.
\begin{claim}
\label{claim:tradeoff}
Let $\calR'$ and $\calR$ be two relations such that $\calR' = \calR\bigcup\{(\bfx,\bfz)\}$, i.e. $\calR'$ allows an extra transformation from $\bfx$ to $\bfz$ than $\calR$. Let $G_{\cal{R}}, G_{\calR'}$ denote their relation graphs and $C_{\bfx, \calR}$ be the weakly connected component (WCC) in $G_{\calR}$ that contains $\bfx$. 
In addition, let $\mu,\eta$ be the same as defined in Sec~\ref{subsec:trade-off}, and $l_{C} = \arg\max_{l\in\calY}\sum_{\bfx\in C}\mu(\bfx)\eta(\bfx,l)$, i.e. the most likely label of inputs in $C$.
Then the change of best attainable robust accuracy from $\calR$ to $\calR'$, denoted by $\Delta_{\calR, \calR'}$, is
\begin{align}
\Delta_{\calR,\calR'} & = Acc^*_{\calR'} - Acc^*_{\calR}\\ 
& = \sum\limits_{\bfx'\in C_{\bfx, \calR'}}\mu(\bfx')\eta(\bfx,l_{C_{\bfx, \calR'}})\notag \\
& \quad - \sum\limits_{\bfx'\in C_{\bfx, \calR}}\mu(\bfx')\eta(\bfx',l_{C_{\bfx, \calR}})\notag \\
& \quad - \sum\limits_{\bfx'\in C_{\bfz, \calR}\backslash C_{\bfx, \calR}}\mu(\bfx')\eta(\bfx', l_{C_{\bfz, \calR}}).
\end{align}
The change $\Delta_{\calR,\calR'} \leq 0$ for all $\calR, \calR'$ pairs, and the equality only holds when $l_{C_{\bfx,\calR}} = l_{C_{\bfz,\calR}}$.
\end{claim}

\begin{proof}
First, we observe that $C_{\bfx,\calR'} = C_{\bfx,\calR}\bigcup C_{\bfz,\calR}$. 
Notice that $(\bfx,\bfz)$ will not change the graph structure outside $C_{\bfx,\calR}\bigcup C_{\bfz,\calR}$: the maximal property of WCC guarantees that neither $\bfx$ nor $\bfz$ connect to nodes outside their own WCC. 
If $C_{\bfx, \calR}$ and $C_{\bfz,\calR}$ are two disjoint WCCs, then the path $\bfx\to_{\calR}\bfz$ will join them to form $C_{\bfz,\calR'}$. Otherwise, $\bfx,\bfz$ are already in the same WCC, and thus $C_{\bfx,\calR'} = C_{\bfx,\calR}=C_{\bfz,\calR}$.

Since the graph structure outside $C_{\bfx,\calR'}$ does not change, it suffices to only look at change of best robust accuracy within $C_{\bfx,\calR'}$. The term $\sum\limits_{\bfx'\in C_{\bfx, \calR'}}\mu(\bfx')\eta(\bfx,l_{C_{\bfx, \calR'}})$ is the best robust accuracy in $C_{\bfx, \calR'}$. When $C_{\bfx, \calR} \neq C_{\bfz,\calR}$, the term $\sum\limits_{\bfx'\in C_{\bfx, \calR}}\mu(\bfx')\eta(\bfx',l_{C_{\bfx, \calR}})$ and $\sum\limits_{\bfx'\in C_{\bfz, \calR}\backslash C_{\bfx, \calR}}\mu(\bfx')\eta(\bfx', l_{C_{\bfz, \calR}})$ are the best robust accuracy in $C_{\bfx,\calR}$ and $C_{\bfz,\calR}$, respectively. When $C_{\bfx, \calR} = C_{\bfz, \calR}$, the latter term becomes zero and the former is the best robust accuracy in $C_{\bfx,\calR}$. In both cases, the equation in Claim~\ref{claim:tradeoff} holds by definition.

Next, we show $\Delta_{\calR, \calR'} \leq 0$. First, consider $l_{C_{\bfx,\calR}} \neq l_{C_{\bfz,\calR}}$. No matter what $l_{C_{\bfx,\calR'}}$ is, it will be different from at least one of $l_{C_{\bfx,\calR}}, l_{C_{\bfz,\calR}}$. Suppose, $l_{C_{\bfx,\calR'}}\neq l_{C_{\bfx,\calR}}$, then the robust accuracy for inputs in $C_{\bfx,\calR}$ will drop. Similarly, the accuracy in $C_{\bfz,\calR}$ will drop if $l_{C_{\bfx,\calR'}}\neq l_{C_{\bfz,\calR}}$. Therefore, $\Delta_{\calR, \calR'} < 0$. Second, when $l_{C_{\bfx,\calR}} = l_{C_{\bfz,\calR}}$, then we will have $l_{C_{\bfx,\calR'}} = l_{C_{\bfx,\calR}} = l_{C_{\bfz,\calR}}$ by definition, and the expression for $\Delta_{\calR,\calR'}$ will evaluate to $0$.
\end{proof}

We also note that if the majority label $l_{C,\calR}$ is not unique for $C_{\bfx,\calR}$ and/or $C_{\bfz,\calR}$, then we consider $l_{C_{\bfx,\calR}} = l_{C_{\bfz,\calR}}$ if any majority label for $C_{\bfx,\calR}$ matches any one for $C_{\bfz,\calR}$.

\subsection{Proof for Proposition~\ref{thm:pspace}}
\label{sec:pspaceproof}
\begin{proposition}[Hardness of Inner Maximization]
\label{thm:pspace}
The inner optimization problem of adversarial training is PSPACE-hard for relational adversaries. 
\end{proposition}
We first write the full statement of Proposition~\ref{thm:pspace} in the following theorem.
\begin{theorem}
\label{claim:pspace}
Let $\calR \subseteq \{ 0,1 \}^d \times \{ 0,1 \}^d$ be a relation. Given a function $f$, an input $\bfx \in \{ 0,1 \}^d$ and a feasible set $\calT(\bfx) = \{\bfz:\bfx\to^*_{\calR}\bfz\}$, solving
the following maximization problem:
\[
\max_{\bfz \in \calT(\bfx)} l(f,\bfz,y)
\]
is PSPACE-hard when $l(f,\bfx,y)$ is the 0-1 classification loss.
\end{theorem}
\begin{proof}
Let $\alpha: \{ 0,1 \}^d \rightarrow \{ 0, 1 \}$ be a
  predicate. Define a loss function $l(f,\bfz,y)$ as follows:
  $l(f,\bfz,y) = \alpha(\bfz)$ (the loss function is essentially the
  value of the predicate). Note that $\max_{\bfz \in \calT(\bfx)}
  l(f,\bfz,y)$ is equal to $1$ iff there exists a $\bfz \in \calT
  (\bfx)$ such that $\alpha(\bfz) = 1$. This is a well known problem
  in model checking called {\it reachability analysis}, which is well
  known to be PSPACE-complete (the reduction is from the problem of checking emptiness
  for a set of DFAs, which is known to be PSPACE-complete~\cite{kozen:pspace}).
\end{proof}
 Recall that the maximation problem $\max_{\bfz \in B_p
   (\bfx,\epsilon) } l(f,\bfz,y)$ used in adversarial training for the image modality was proven to be NP-hard~\cite{reluplex}.
 Hence it seems that the robust optimization problem in our context
 is in a higher complexity class than in the image domain.

\subsection{Model Capacity Requirement}
\label{subsec:modelcapcity}
In this section, we illustrate how the \nandp framework can potentially help reduce the model complexity for learning a robustly accurate classifier. We start with the following proposition.

\begin{proposition}[Model Capacity Requirement]
\label{thm:modelcapacity-full}
For some hypothesis class $\calH$ and relation $\calR$, there exists $f\in\calH$ such that $f(\calN_{\calR}(\cdot))$ is robustly accurate, but no $f\in\calH$ can be robustly accurate on the original inputs.  
In other words, robustly accurate classifier can only be obtained after normalization.
\end{proposition}

We first define an equivalence relation induced by equivalent coordinates over binary inputs, and then write the formal statement of the observation in the following claim.

\begin{definition}[Equivalence relation induced by equivalent coordinates]
\label{eg:equiv}
Let $\bfx = (\bfx_1, \cdots, \bfx_d)$ be a binary input vector on $\{0,1\}^d$, where each $\bfx_i, i\in\{1,\cdots,d\}$ is a coordinate.
Let $I = \{1,\cdots,d\}$ be the set of coordinate indices for inputs in $\calX$ and $U = \{i_1, \cdots, i_m\} \subseteq I$. In an equivalence relation $\calR$ induced by $U$, $\bfx \to_{\calR} \bfz$ \emph{iff} 1) $\bfx_i = \bfz_i$ for all $i \in I\backslash U$, and 2) $\bigvee_{i\in U}\bfx_i = \bigvee_{i\in U}\bfz_i$. Notice that $\bfx \to_\calR \bfz$ iff $\bfz \to_{\calR}\bfx$. 
\end{definition}
The notation $\bigvee_{i\in U}\bfx_i$ means taking a {\it logic or} operation over all $\bfx_i$s for $i\in U$.
Intuitively, having an equivalence relation induced by coordinates with indices in $U$ means the presence of any combination of such coordinates is equivalent to any other combination. Usage of interchangeable APIs in malware implementation is an example of equivalence relation: the attacker can choose any combination from a set of equivalent APIs to implement the same functionality.    

In Definition~\ref{eg:equiv}, we use $U$ to represent the set of \emph{indices} of the equivalent coordinates. In the following theorems and proofs, we overload $U$ to also represent the set of equivalent coordinates directly, and the $\bigvee$ operation will be taken over all coordinates in $U$. 
\begin{claim}
\label{claim:modelcapacity}
Consider $\calX = \{0,1\}^5$ and $\calY=\{0,1\}$. Let the coordinates of an input $\bfx\in\calX$ be $\{\bfx_1, \bfx'_1, \bfx_2, \bfx_3, \bfx_4\}$. Suppose we have an equivalence relation induced by $U = \{\bfx_1,\bfx'_1\}$. Meanwhile, the true label $y$ of an input $\bfx$ is $1$ \emph{iff} any of the following clauses is true: 1) $(\bfx_2=1)\wedge(\bfx_3 = 1)$, 2) $(\bfx_1\vee\bfx'_1=1)\wedge(\bfx_2=1)$, 3) $(\bfx_1\vee\bfx'_1=1)\wedge(\bfx_3=1)\wedge(\bfx_4=1)$. Then
\begin{enumerate}
\item no linear model can classify the inputs with perfect robust accuracy, but
\item a robust and accurate linear model exists under \nap.
\end{enumerate}
\end{claim}

\begin{proof}
Let $\calH = \{f_{\bfw,b}: sgn(\langle\bfw,\bfx\rangle+b)\}$.
Let $\bfw_1, \bfw'_1, \bfw_2,\bfw_3,\bfw_4$ denote the coordinates in $\bfw$ that corresponds to $\bfx_1, \bfx'_1, \bfx_2, \bfx_3, \bfx_4$. 

We know $y=1$ if $\bfx_1 = 1, \bfx_2 = 1$ and the other coordinates are zero because the second clause in the labeling rule is satisfied.
Therefore, in order to classify this input instance correctly, we must have $\bfw_2+\bfw_1+b >0$. Since $\bfx_1$ and $\bfx'_1$ are equivalent, we should also have
$\bfw_2+\bfw'_1+b >0$.

Similarly, we know $y=0$ if $\bfx_2=1, \bfx_4=1$ and the other coordinates are zero because none of the clauses are satisfied.
Therefore, we must have $\bfw_2+\bfw_4+b<0$. 

In order to classify all possible $\bfx$ correctly, the classifier $f_{\bfw,b}$ must satisfy
\begin{align}
\bfw_2 + \bfw_4 + b < 0 \label{eq:1}\\
\bfw_2 + \bfw_1 + b > 0 \label{eq:2}\\
\bfw_2 + \bfw'_1 + b > 0 \label{eq:3}\\
\bfw_3 + \bfw_1 + \bfw'_1 + b < 0 \label{eq:4}\\
\bfw_3 + \bfw_4 + \bfw_1 + b > 0 \label{eq:5}
\end{align}
First, by Formula~\ref{eq:1},~\ref{eq:2} and~\ref{eq:3}, we have $\bfw_1 > \bfw_4$ and $\bfw'_1 > \bfw_4$.
However, by Formula~\ref{eq:4} and~\ref{eq:5}, we have $\bfw_1+\bfw'_1 < \bfw_4 + \bfw_1$,
which implies $\bfw'_1<\bfw_4$. Contradition. Therefore, no linear classifier can satisfy all the equations.

On the other hand, if we perform normalization by letting $\bfx_1 = \bfx_1\vee\bfx'_1$ and removing $\bfx'_1$, then a classifier $f_{\bfw,b}$ -- with $\bfw_1 = 0.4, \bfw_2=0.7, \bfw_3 = 0.5, \bfw_4 = 0.2, b = -1$ -- can perfectly classify $\bfx$.
\end{proof}

\subsection{Proof for Theorem~\ref{thm:equiv}}
\label{sec:equivproof}
We first write down the full formal statement.
\begin{theorem}[Preservation of robust accuracy]
\label{thm:equiv-full}
Consider $\calH$ to be the set of all labeling functions on $\{0,1\}^d$.
Let $f^*$ be the classifier that minimizes the objective of our unified framework over data distribution $\calD$, i.e. the optimal solution to
$$
\min\limits_{f\in\calH}\max\limits_{A(\cdot)}\bbE_{(\bfx,y)\sim \calD}\ell\left(f,A\left(\calN_{\calR'}(\bfx)\right), y\right),
$$
where $\ell$ is the 0-1 classification. Meanwhile, let $f^*_{adv}$ be the classifier that minimizes the objective of adversarial training over $\calD$, i.e.
the optimal solution to 
$$
\min\limits_{f\in\calH}\max\limits_{A(\cdot)} \bbE_{(\bfx,y)\sim \calD}\ell(f, A(\bfx), y).
$$
Then in principle, $f^*(\calN_{\calR'}(\cdot))$ and $f^*_{adv}$ have the same optimal robust accuracy if $\calR'$ is reversible.
\end{theorem}

\begin{proof}
We know by definition that $f^*_{adv}$ has the highest possible robust accuracy. Therefore, it suffices to show that there exists a classifier $f(\calN_{\calR'}(\cdot))$ under our unified framework that has at least as good robust accuracy as $f^*_{adv}$.

We consider an $f$ under our unified framework as follows.
Let $C$ denote a connected component (which is also strongly connected for reversible relation) $C$ in $G_{\calR'}$ and $\calY_{C}$ be the set of all labels that $f^*_{adv}$ assign to inputs in $C$, i.e. $\calY_{C} = \{y\in\calY|\exists \bfx\in C\ s.t.\ f^*_{adv}(\bfx)=y \}$. Let $C_{\bfx}$ denote the connected component which contains $\bfx$. We consider a classifier $f$ such that
\begin{equation}
\label{eq:f}
f(\bfx) = \arg\max_{l\in\calY_{C_{\bfx}}} \bbE_{(\bfz,y)\sim \calD}\mathds{1}(\bfz\in C_{\bfx} \wedge y = l),
\end{equation}
i.e. $f$ predicts the same label for all inputs in a connected component, and chooses the label that maximizes the clean accuracy in the connected component. 
Notice that $f(\bfx) = f(\calN_{\calR'}(\bfx))$ by construction, so we will use $f(\bfx)$ instead of $f(\calN_{\calR'}(\bfx)$ for simplicity.

Recall that we say a classifier $f$ is robustly accurate over an input-label pair $(\bfx, y)$ under $\calR$ if and only if $f(\bfz)=y$ for all $\bfz:\bfx\to^*_{\calR}\bfz$. Now to show that $f$ with the prediction rule in Equation~\ref{eq:f} has the same robust accuracy as $f^*_{adv}$, we prove the following claim.
\begin{claim}
\label{thm:helpclaim}
The classifier $f$ is robustly accurate over $(\bfx,y)$ under $\calR$ if $f^*_{adv}$ is robustly accurate $(\bfx,y)$ under $\calR$.
\end{claim}
We prove the contrapositive of the claim. Suppose $f$ is not robustly accurate over $(\bfx,y)$. 
Let $\bfz$ denote the input such that $\bfx \to^*_{\calR}\bfz$ and $f(\bfz)\neq y$. 


First, by the definition of $f(\bfz)$, we know that there must be a $\bfz'\in C_{\bfz}$ such that $f^*_{adv}(\bfz') = f(\bfz) \neq y$ because $f$ only predicts the label that has been used by $f^*_{adv}$ over some input in the same connected component. Since $C_{\bfz}$ is strongly connected by our initial assumption, we must have $\bfz\to^*_{\calR'} \bfz'$.

Now, because $\bfx\to^*_{\calR}\bfz$ and $\bfz\to^*_{\calR'} \bfz'$, we have $\bfx\to^*_{\calR}\bfz'$. Since $\bfx\to^*_{\calR}\bfz'$ but $f^*_{adv}(\bfz') \neq y$, we can conclude that $f^*_{adv}$ is not robustly accurate at $\bfx$, and thus prove the claim. 

Claim~\ref{thm:helpclaim} suggests that $f$ has at least the same robust accuracy as $f^*_{adv}$. Since $f(\bfx) = f(\calN_{\calR'}(\bfx))$ by our construction of $f$, we know $f(\calN_{\calR'}(\cdot))$ has at least the same robust accuracy as $f^*_{adv}$ too. Moreover, $f^*(\calN_{\calR'}(\cdot))$, defined as the classifier with the highest robust accuracy after normalizaiton, should have at least the same robust accuracy as $f(\calN_{\calR'}(\cdot))$. Therefore, $f^*(\calN_{\calR'}(\cdot))$ has at least the same robust accuracy as $f^*_{adv}$. On the other hand, we know $f^*_{adv}$ has the highest possible robust accuracy by definition. Therefore, we can conclude that $f^*(\calN_{\calR'}(\cdot))$ and $f^*_{adv}$ have the same robust accuracy.
\end{proof}

\subsection{Computation Efficiency of Normalizer for Reversible Relation}
\label{subsec:complexity}

We show in Proposition~\ref{lem:complexity} that normalizing over reversible relation is at most as hard as finding the strongest adversarial example, which is required for optimal adversarial training.
\begin{proposition}
\label{lem:complexity}
Let $(\bfx, y)$ denote a sample point where $\bfx$ is the input vector and $y$ is the ground truth label. Let $\calR'$ denote a reversible relation, and $\calT(\bfx) = \{\bfz|\bfx\to^*_{\calR'}\bfz\}$ denote the feasible set of adversarial examples. Let $\bfz^* = \arg\max_{\bfz\in\calT(\bfx)} \ell(f,\bfz,y)$, i.e. the most powerful adversarial example of $\bfx$ for model $f$, then the normalizer $\calN(\bfx)=\bfz^*$ satisfies the robustness condition in Proposition~\ref{thm:robustcondition}.~\footnote{A mild assumption is that the argmax solver has a consistent tie-breaker in case multiple optima exist.}
\end{proposition}

Let $C_i$ be any connected component in $G_{\calR'}$. Since $\calR'$ is reversible, $C_i$ shall be strongly connected. The strongest adversarial example to an input $\bfx\in C_i$ shall be the same for all nodes in $C_i$. Meanwhile, the strongest adversarial example is deterministic given a consistent tie-breaker between adversarial choices with the same test loss. Therefore, $\calN(\bfx) = \bfz^*$ satisfies the robustness condition. 

\subsection{Semantics-Preserving Transformation and Most Likely Label}
\label{subsec:labels}
In Figure~\ref{fig:tradeoff}, one may notice that two inputs $\bfx$ and $\bfz$ may have different most likely labels even though $\bfz$ is obtained from $\bfx$ by a semantics-preserving transformation. This phenomenon may look bizarre, but is indeed possible in real-world settings. The main reason is that the input vector may contain information that (1) is irrelevant to the essential semantics for classification task, and yet (2) may enhance classification accuracy on clean inputs.

Taking malware detection as an example. A group of zip tools may have the same static feature vector $\bfx$ as some ransomware. Authors of zip tools may agree to use a secret syntactic pattern with no semantic implication in an attempt to distinguish from ransomwares circulated on the web. Let $\bfz$ be the static feature vector of zip tools after the pattern is added. Now given the natural distribution and the absence of relational adversaries, $\bfz$ will certainly be classified as benign zip tools, while $\bfx$ is predominantly ransomware.

However, in the long run, an adaptive malware author will eventually know this secret pattern, and also add it to his ransomware. Similarly, the secret pattern may also be gradually abandoned by authors of zip tools as time goes by making the zip tools identical to the ransomware.
In short, $(\bfx,\bfz)$ is a semantics-preserving transformation whose source $\bfx$ and target $\bfz$ do not have the same most likely label on the clean distribution.
\end{document}